\documentclass[letterpaper]{article} 
\usepackage{aaai2027}
\nocopyright 

\usepackage[hyphens]{url}  
\usepackage{graphicx} 
\usepackage{natbib}  
\usepackage{caption} 
\usepackage{algorithm}
\usepackage{algorithmic}

\usepackage{newfloat}
\usepackage{listings}
\DeclareCaptionStyle{ruled}{labelfont=normalfont,labelsep=colon,strut=off} 
\floatstyle{ruled}
\newfloat{listing}{tb}{lst}{}
\floatname{listing}{Listing}

\usepackage{booktabs}

\usepackage{amsmath, amssymb}

\usepackage{amsmath,amsfonts}

\def\eqref#1{equation~\ref{#1}}
\def\Eqref#1{Equation~\ref{#1}}

\def\1{\bm{1}}

\DeclareMathAlphabet{\mathsfit}{\encodingdefault}{\sfdefault}{m}{sl}
\SetMathAlphabet{\mathsfit}{bold}{\encodingdefault}{\sfdefault}{bx}{n}

\usepackage{tikz}
\usepackage{pgfplots}
\pgfplotsset{compat=1.18}

\usepackage{amsthm}

\usepackage{xcolor}
\usepackage{multirow}
\usepackage{cleveref}
\usepackage{booktabs, makecell}
\usepackage{subcaption}

\newtheorem{lemma}{Lemma}
\newtheorem{theorem}{Theorem}

\providecommand{\Eqref}[1]{\Cref{#1}}

\title{Score-Based Stabilization for Time-Dependent Problems}

\author{
    Eshed Gal, Eldad Haber, Uri Ascher
}
\affiliations{
    University of British Columbia\\
    Vancouver, British Columbia, Canada\\
    eshedg@cs.ubc.ca, ehaber@eoas.ubc.ca, ascher@cs.ubc.ca
}

\begin{document}

\maketitle

\begin{abstract}
We propose a score-based stabilization framework for numerical simulation of partial differential equations, in which a learned score model defines a stabilization operator applied to provisional numerical updates. This operator augments standard time-stepping schemes by enforcing structure and physical consistency through a correction that drives iterates toward the manifold of admissible states. We show that the stabilization operator acts as a contraction toward this manifold, yielding a correction mechanism with basin-conditional stability. Numerical experiments on Advection, Korteweg–de Vries (KdV), Nonlinear Schrödinger (NLS), and Burgers’ equations demonstrate improved robustness, suppression of nonphysical instabilities, and preservation of qualitative dynamics.
\end{abstract}


\section{Introduction}

The numerical solution of nonlinear time-dependent partial differential equations proves challenging, particularly in regimes where nonlinear effects are strong and fine resolution is costly. Time-stepping methods face a fundamental trade-off between stability, accuracy, and computational efficiency. Moreover, for many long-time simulations of physical systems, maintaining numerical stability and preserving the structural properties of the solution are often more important than minimizing the local truncation error. In particular, preserving physical invariants such as mass, momentum, energy, or Hamiltonian structure is frequently essential for obtaining physically meaningful simulations over long time horizons \citep{lax1973hyperbolic,leveque2002finite,hlw}. Explicit high-order schemes often require very small timesteps, while implicit methods are more computationally expensive per step.

Many numerical methods for for time-stepping can be viewed as a two-stage procedure. First, the solution is updated, and may become unstable when large timesteps are used. Second, a stabilization step is applied to reduce nonphysical effects, but might come at the cost of losing important solution properties \citep{boris1973flux}.

The challenge is especially significant for time dependent PDEs. In equations such as the Korteweg-de Vries (KdV) equation \citep{Ascher2004Multisymplectic, Ascher2005Symplectic}, conservative non-damping discretization methods coupled with insufficient resolution high-order derivatives restrict the allowable timestep for explicit methods, and insufficient resolution of nonlinear interactions can lead to instability over long-time behavior.


\begin{figure}[t]
  \centering
  \includegraphics[width=\linewidth]{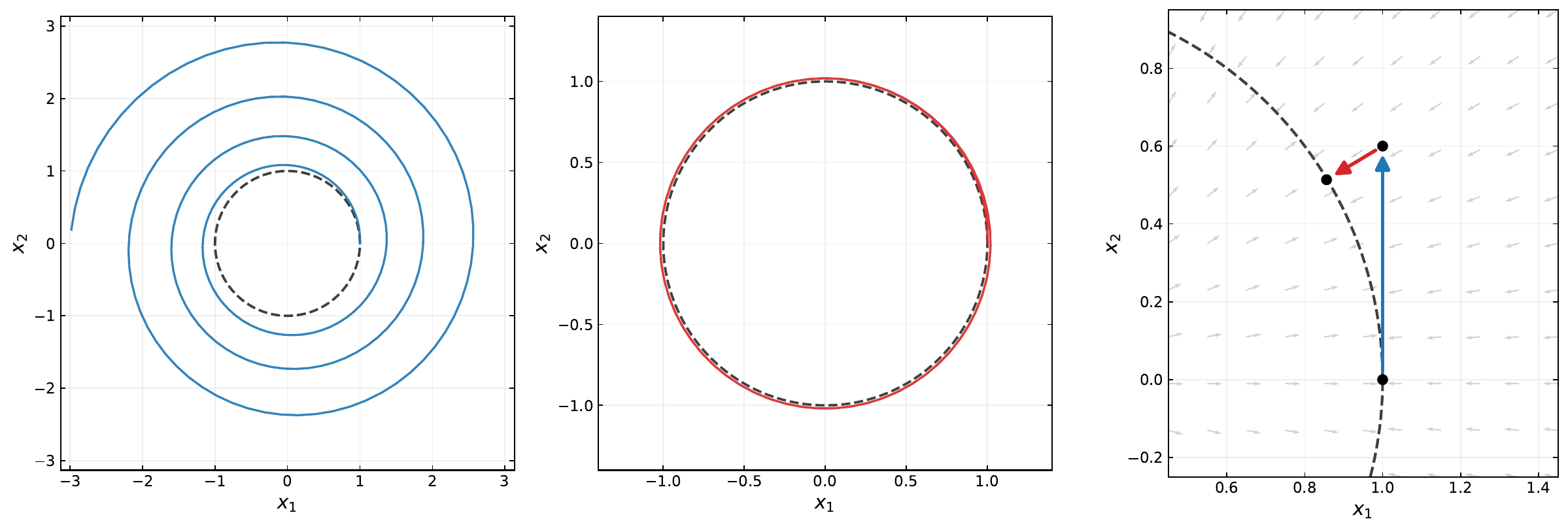}
  \caption{%
    Score-based correction on a unit-circle model problem.
    \textbf{Left:} explicit forward Euler scheme which spirals steadily outward and drifts off the manifold.
    \textbf{Center:} score-corrected scheme for the same step size.
    \textbf{Right:} a single step illustrating the
    mechanism. The gray arrows
    show the score field $s(x)$, and the dashed curve represents the exact solution.%
  }
  \label{fig:circle-overview}
\end{figure}

Recent advances in machine learning, particularly score-based generative modeling, provide new tools for representing complex high-dimensional probability distributions. Central to these methods is the score function \citep{Song2021ScoreBasedSDE}, defined as the gradient of the log-density of a target distribution. Score-based corrections are widely used in denoising and sampling \citep{luo2021score}.

This suggests an alternative to standard stabilization methods in an attempt to avoid artificial damping of the computed solution. If the set of physically meaningful states can be described or approximated by a probability density function, then, a numerical update can be adjusted using a probabilistic correction. Instead of introducing heuristic stabilization terms, the correction is informed by the learned manifold of physically admissible solution states. 
Unlike previous machine learning approaches that primarily seek to improve point-wise prediction accuracy \citep{kochkov2021machine, bar2019learning}, the objective of the score model here is stabilization without losing important solution characteristics. Rather than correcting individual numerical errors, it acts to preserve the qualitative character of the solution, steering unstable numerical trajectories back toward physically meaningful states and ensuring that the numerical evolution remains faithful to the underlying physical dynamics.

While learning the score function requires an offline training stage,
this cost is amortized across many simulations. Many scientific computing
applications such as weather prediction, repeatedly solve the same governing PDE under different
initial conditions. Since the learned score
captures the manifold of admissible states rather than a single trajectory,
the same model can be reused across an entire family of simulations.
This enables inexpensive numerical schemes to operate stably in regimes
where they would otherwise fail, avoiding the need for more costly
integrators.


In this work, we therefore explore the score-based stabilization idea by augmenting standard time-stepping schemes with a learned score-based correction applied at each step. An illustration of an example for our proposed method is presented in \Cref{fig:circle-overview}, with a detailed description in \Cref{sec:simple-example}. The objective is not to replace the underlying numerical solver altogether, but to stabilize it in a principled manner that permits larger timesteps while preserving the structural character of the solution and the exact conservation of some critical physical invariants.

\paragraph{Contributions.}
The contributions of this work are:
\begin{itemize}
\item We propose a novel \emph{two-stage stabilization framework} that augments standard time-stepping numerical schemes with a learned, probabilistic score-based correction.

\item We demonstrate that this score correction acts as a versatile stabilization layer, \emph{successfully mitigating the nonphysical behaviors and/or numerical instabilities} that cause classical integrators to fail or blow up in unstable regimes.

\item We show that the score-corrected scheme achieves \emph{conditional manifold stability} with respect to the solution manifold, if there is one. Specifically, we demonstrate that the stability bound is independent of the number of time steps and circumvents traditional strict timestep restrictions, provided the provisional update remains within the valid basin of attraction of the score correction.

\item We validate our approach across a \emph{diverse set of initial-boundary value PDEs}, including Advection, Korteweg--de Vries (KdV), Nonlinear Schrödinger (NLS), and Burgers' equations.


\end{itemize}

\section{Related Work}

\textbf{Numerical Stabilization and Two-Stage Methods.} 
The challenge of suppressing numerical artifacts and avoiding numerical blow-up in the integration of nonlinear partial differential equations has a long history. Classical stabilization techniques often rely on injecting artificial viscosity or using damping numerical discretizations to dampen high-frequency oscillations that arise near shocks or steep gradients \citep{vonneumann1950method}. Modern extensions of this idea evolved into Total Variation Diminishing (TVD) schemes \citep{harten1997high} and Weighted Essentially Non-Oscillatory (WENO) methods \citep{jiang1996efficient}, which adaptively manage local numerical diffusion. Other classical approaches rely on spectral filtering \citep{gottlieb2001spectral} to selectively damp or truncate high-frequency modes in pseudo-spectral environments. Such schemes frequently adopt a two-stage architecture. In the first stage, a provisional baseline update is computed, which may introduce numerical dispersion or push the state off a governing manifold. In the second stage, a stabilization operator—such as a flux limiter, spatial filter, or structural projection—is applied to enforce physical constraints or smooth nonphysical ripples \citep{maccormack2003effect}. While effective, these hand-crafted fixes inherently trade off accuracy for stability, often resulting in the loss of important, physically meaningful solution properties, such as excessive dissipation of fine-scale features or physical invariants \citep{godunov1959difference, leveque2002finite}.

\textbf{Score Models and Score Functions.} 
Recent advancements in generative modeling have established score-based frameworks as state-of-the-art for representing high-dimensional data distributions \citep{ho2020denoising}. Central to these frameworks is the score function, which describes the gradient of the log-density of the data \citep{Song2021ScoreBasedSDE}. By training neural networks via score matching to estimate this vector field, corrupted states can be iteratively denoised and guided back toward the target distribution. This process is inherently tied to the noise-level relation: score models are typically conditioned on varying levels of Gaussian noise, which dictate the magnitude and structure of the correction. High noise levels pull severely corrupted states from broad domains toward the general manifold, while lower noise levels refine the local, high-frequency structure \citep{meng2021sdedit}. In our context, we leverage this learned vector field not to generate images, but to push an unstable numerical state back onto the manifold of physically valid PDE solutions. This mirrors recent efforts to adapt diffusion models for Riemannian manifolds and constrained physical systems, ensuring trajectories remain physically consistent without relying on hand-crafted heuristics \citep{de2022riemannian}.

\textbf{Machine Learning for Numerical Stability.} 
Beyond generative corrections, there is a growing body of work leveraging machine learning to tackle numerical instability and quantify solver errors. Previous work on data-driven discretizations \citep{bar2019learning} and hybrid machine-learning solvers \citep{kochkov2021machine} has demonstrated that neural networks can actively correct coarse or unstable numerical steps. Similarly, Physics-Informed Neural Networks (PINNs) integrate governing equations directly into the learning process to ensure physical consistency \citep{raissi2019physics}. Parallel to these deterministic correctors, probabilistic frameworks have emerged to model the uncertainties introduced by standard numerical discretizations. Work by \citet{matsuda2021estimation} treats discretization errors as random variables, allowing their variances to be updated concurrently with system parameters. Building on this, \citet{miyatake2025quantifying} imposed a monotonicity constraint on discretization error variances using Bayesian isotonic regression. Further advancements address these inaccuracies through joint parameter evaluation \citep{toyota2025joint}, demonstrating how probabilistic techniques can dynamically estimate and correct the error trajectories induced by numerical solvers.

\section{Preliminaries}

This section presents the framework for the two-stage stabilization viewpoint and the associated score-based correction.

\subsection{Stabilization as a Two-Stage Numerical Process}

Consider a nonlinear equation of the form
\begin{equation}
\label{eq:pde}
\frac{d x}{d t} = \mathcal{N}(x),
\end{equation}
where the vector function $x(t)$ denotes the state of the system at a time $t$ and $\mathcal{N}$ is a nonlinear operator arising, for example, from spatial discretization of a PDE.

Considering a time step of length $h$ from $t_k$ to $t_{k+1} = t_k + h$, many numerical schemes for \Eqref{eq:pde} can be abstractly written as a two-stage method,
\begin{align}
\label{eq:two_step}
\hat{x}_{k+1} &= G(x_k), \\
x_{k+1} &= F(\hat{x}_{k+1}),
\end{align}
where $G$ is a time integration operator, and $F$ is a stabilization operator.
The update $\hat{x}_{k+1}$ is typically stable only when timestep restrictions are respected. When these restrictions are violated, $\hat{x}_{k+1}$ may exhibit nonphysical behavior such as blowup, oscillations or violation of invariant constraints. The role of $F$ is to mitigate such effects.


\subsection{Score-Based Methods}\label{sec:score_corrections}

Score-based modeling introduces the \emph{score}, which characterizes local changes of the
log-density in state space \citep{Hyvarinen2005ScoreMatching}. Let $p(x)$ denote a probability density over admissible states
$x\in\mathbb{R}^d$. The score is defined as
\begin{equation}\label{eq:score_def}
s(x) = \nabla_x \log p(x).
\end{equation}

\paragraph{Score points toward higher probability.}
Define the negative log-density
\begin{equation}\label{eq:energy_def}
E(x) \;=\; -\log p(x),
\end{equation}
where $p(x) > 0$, so that higher probability corresponds to smaller $E(x)$. Since
\begin{equation}\label{eq:grad_energy}
\nabla E(x) \;=\; -\nabla \log p(x) \;=\; -s(x),
\end{equation}
a standard gradient descent step to \emph{decrease} $E$ is
\begin{equation}\label{eq:gd_energy}
x_{\text{new}} \;=\; x - \eta \nabla E(x) \;=\; x + \eta\, s(x), \qquad \eta>0.
\end{equation}
Thus, the score $s(x)$ provides a local direction that moves the state toward regions where $p(x)$ is larger, for stepsize $\eta$ sufficiently small.

\paragraph{Gaussian perturbations and the score correction.}
In many applications of score-based methods, it is possible to model a perturbed state as
\begin{equation}\label{eq:noise_model}
\hat{x} = x + \varepsilon, \qquad \varepsilon \sim \mathcal{N}(0,\sigma^2 I),
\end{equation}
where $x$ is an (unknown) admissible state and $\hat{x}$ is an observed or computed perturbation.
A Gaussian model provides a simple local approximation to aggregated small errors.

Under \Eqref{eq:noise_model}, the posterior mean satisfies the identity (Tweedie's formula \citep{Efron2011Tweedie})
\begin{equation}\label{eq:tweedie_statement}
\mathbb{E}[x \mid \hat{x}]
\;=\;
\hat{x} + \sigma^2 \nabla_{\hat{x}} \log p(\hat{x})
\;=\;
\hat{x} + \sigma^2 s(\hat{x}),
\end{equation}
where $p(\hat{x})$ denotes the density of the perturbed variable induced by $p(x)$ and Gaussian noise.
Thus, $\sigma^2 s(\hat{x})$ gives a principled correction direction, and the factor $\sigma^2$
sets the correction magnitude: larger perturbations (larger $\sigma$) imply a stronger shift.

\paragraph{ODE viewpoint.}
The update in \Eqref{eq:gd_energy} can be viewed as taking one step in the direction of the vector field $s(x)$.
A convenient way to formalize this is to introduce the ODE
\begin{equation}\label{eq:score_flow}
\frac{dx}{d\tau} = s(x) = \nabla_x \log p(x),
\end{equation}
whose solutions move in the score direction. Discretizing \Eqref{eq:score_flow} using forward Euler gives
\begin{equation}\label{eq:euler_score}
x(\tau+\eta) \approx x(\tau) + \eta\, s(x(\tau)),
\end{equation}
which has exactly the same form as \Eqref{eq:gd_energy}. This well-known  ODE interpretation is useful for
describing the correction as a short score-driven evolution applied after each numerical time step.

\section{Method}
\label{sec:score_method}

We describe the stabilization strategy for time-dependent nonlinear equations that augments a standard time integrator with a score-based correction. The aim is to improve stability and permit larger timesteps while preserving the structure of the base solver.

\subsection{Algorithmic Formulation}

Let $G_{\Delta t}$ denote a chosen baseline time-stepping operator
with timestep $\Delta t$. Given the current state $x_k$, the proposed
method consists of two stages:

\paragraph{1. Provisional update.}
Compute a standard numerical step
\begin{equation}
\hat{x}_{k+1}=G_{\Delta t}(x_k).
\label{eq:predictor}
\end{equation}

\paragraph{2. Stabilization layer.}
Apply a stabilization operator
\begin{equation}
x_{k+1}=\mathcal{S}(\hat{x}_{k+1}),
\label{eq:stabilization}
\end{equation}
where $\mathcal{S}$ is built around a learned score function.

In its simplest form, the stabilization layer consists of a single
score correction,
\begin{equation}
\mathcal{S}(x)=x+\sigma^2 s(x),
\label{eq:score_update}
\end{equation}
where
$s(x)=\nabla\log p(x)$
and $\sigma>0$ represents the effective magnitude of the perturbations
introduced by the provisional update. The parameter $\sigma$
controls the strength of the stabilization.
More generally, $\mathcal{S}$ may incorporate additional
structure-preserving operations whenever they are available.
For example, if the underlying PDE possesses known invariants such as
mass, momentum, or Hamiltonian, the stabilization layer may combine
the learned score correction with exact projections that enforce these
conservation laws.

The underlying numerical solver $G_{\Delta t}$ remains unchanged, and the stabilization operator $\mathcal{S}$ is applied as a post-processing map acting on the provisional state. 
In practice, the score function $s(x)$ is approximated by a neural network trained offline using samples of physically admissible states. At inference time, no additional optimization or sampling procedure is required, and stabilization consists solely of evaluating the learned score field and applying the resulting correction within $\mathcal{S}$.

\subsection{Score Model and Training Objective}
To implement the stabilization operator, we train a neural network $s_\theta(x)$ to approximate the score function of a noise-smoothed version of the target distribution $p(x)$ over the manifold of admissible states. Under this setup, the model learns a noise-smoothed/denoising score function.

Let $x$ represent a physically admissible state drawn from the target distribution $p_{\mathrm{data}}(x)$, and let $\tilde{x} = x + \varepsilon$ represent a perturbed state, where $\varepsilon \sim \mathcal{N}(0, \sigma^2 I)$. The denoising score matching (DSM) objective function used to train the parameters $\theta$ of our neural network $s_\theta$ is defined as:
\begin{equation}
\begin{split}
\mathcal{L}_{\mathrm{DSM}}(\theta) &= \mathbb{E}_{\sigma \sim p(\sigma)} \mathbb{E}_{x \sim p_{\mathrm{data}}(x)} \mathbb{E}_{\varepsilon \sim \mathcal{N}(0, \sigma^2 I)} \\
&\quad \left[ \lambda(\sigma) \left\| s_\theta(x + \varepsilon, \sigma) - \nabla_{\tilde{x}} \log p(\tilde{x} \mid x) \right\|_2^2 \right],
\end{split}
\label{eq:dsm_loss}
\end{equation}
where $\lambda(\sigma) = \sigma^2$ is a weight function that balances the loss across scales, and the conditional score can be computed analytically as:
\begin{equation}
\nabla_{\tilde{x}} \log p(\tilde{x} \mid x) = -{(\tilde{x} - x)}{\sigma^{-2}} = -{\varepsilon}{\sigma^{-2}}.
\label{eq:cond_score}
\end{equation}
Substituting \Eqref{eq:cond_score} into \Eqref{eq:dsm_loss} yields the explicit training objective:
\begin{equation}
\begin{split}
\mathcal{L}_{\mathrm{DSM}}(\theta) &= \mathbb{E}_{\sigma \sim p(\sigma)} \mathbb{E}_{x \sim p_{\mathrm{data}}(x)} \mathbb{E}_{\varepsilon \sim \mathcal{N}(0, \sigma^2 I)} \\
&\quad \left[ \left\| \sigma s_\theta(x + \varepsilon, \sigma) + {\varepsilon}{\sigma^{-1}} \right\|_2^2 \right].
\end{split}
\label{eq:dsm_loss_explicit}
\end{equation}
By minimizing this objective, the network $s_\theta(\tilde{x}, \sigma)$ learns to predict the noise-smoothed score \(\nabla_{\tilde{x}} \log p_\sigma(\tilde{x})\). At inference time, the learned score model acts as a vector field pointing from perturbed or unstable states back toward the high-density manifold of valid physical solutions.

\subsection{A Simple Example}
\label{sec:simple-example}

To visually illustrate our method, consider a simple state \(x = (x_1, x_2) \in \mathbb{R}^2\) evolving under the linear harmonic oscillator
\begin{equation}
  \frac{dx}{dt} = A\,x,
  \qquad
  A = \begin{pmatrix} 0 & -1 \\ 1 & 0 \end{pmatrix}.
  \label{eq:circle-ode}
\end{equation}
Taking \(\lVert x(0) \rVert = 1\), the exact solution lives on the unit-circle manifold:
\begin{equation}
  \mathcal{M} = \bigl\{\, x \in \mathbb{R}^2 : \lVert x \rVert^2 = 1 \,\bigr\}.
  \label{eq:circle-manifold}
\end{equation}

We associate with $\mathcal{M}$ the potential energy $U(x) = \tfrac{1}{4}\bigl(\lVert x \rVert^2 - 1\bigr)^2$, defining a score field that points toward the manifold:
\begin{equation}
  s(x) = -\nabla U(x) = -x\,\bigl(\lVert x \rVert^2 - 1\bigr).
  \label{eq:circle-score}
\end{equation}

Discretizing \Eqref{eq:circle-ode} via the forward Euler scheme gives
$
  x^{k+1} = (I + hA)\,x^{k}
  \label{eq:circle-euler}
$.
Since the discrete radius grows as $\lVert (I+hA)\,x \rVert = \sqrt{1+h^2}\,\lVert x \rVert$ for any step size $h>0$, the uncorrected numerical solution inevitably spirals outward.

Our scheme augments this unstable predictor with a score correction step to guide the state back to $\mathcal{M}$. \Cref{fig:circle-overview} contrasts the two approaches: the plain forward Euler trajectory drifts rapidly off the manifold (\emph{left}), while the score-corrected scheme remains stable on the unit circle (\emph{center}). This stability of the score correction neutralizes the tangential overshoot introduced by the baseline predictor (\emph{right}).

\subsection{Training Data}

The score model requires only samples from the target distribution $p$ of
physically admissible states and does not require a closed-form solution or
paired coarse/fine trajectories. This distribution is a property of the
governing equation and the physical regime rather than any particular initial
condition, allowing a single trained model to characterize the admissible
state manifold for an entire family of simulations.

In many practical settings, suitable training data already exist in the form
of high-resolution simulations generated for parameter studies or uncertainty
quantification. When such data is unavailable,
training samples can be generated offline using a trusted high-fidelity
solver over a representative ensemble of initial conditions.
Once trained, the score model can be reused for different initial guesses. During
deployment, stabilization requires only a single evaluation of the learned
score field at each correction step, making the online computational overhead
small.

\subsection{Mathematical Analysis}
\label{sec:math_analysis}

We analyze the stability properties of the proposed stabilization framework. Let $X$ be a Hilbert space and let $\mathcal{M}\subset X$ denote the manifold of physically admissible states. We aim to show that the two-stage scheme maintains the state close to $\mathcal{M}$ without sacrificing the accuracy of the baseline solver. 

To establish this result, we first prove three lemmas. First, we establish that applying the score-based correction does not degrade the underlying order of accuracy of the base numerical integrator. Second, we bound the geometric deviation introduced by the uncorrected provisional step, characterizing how far the baseline solver pulls the state away from the manifold in a single step. Finally, we guarantee that a gradient-based correction acts as a strong contraction, actively pulling the provisional state back toward the valid distribution.

Using these lemmas, we then analyze the global behavior of the stabilized scheme. We demonstrate that if the numerical step remains within the correction's basin of attraction, the two-stage method achieves bounded manifold distance uniformly over time.

\begin{theorem}[Conditional Manifold Stability]
Let
\[
r_j:=\operatorname{dist}_X(u_j,\mathcal{M})
\]
denote the constraint residual, and let $\varepsilon>0$. Let $\rho_K$ be
the contraction factor from Lemma~3. Assume that whenever
$r_j\le\varepsilon$, the predictor stays in the valid basin of attraction
of the score correction and satisfies
\[
r_{\widehat{j+1}}
\le L_h\varepsilon+Ch^q,
\qquad
r_{\widehat{j+1}}
:=
\operatorname{dist}_X(\widehat{u}_{j+1},\mathcal{M}).
\]
Assume further that
$
\rho_K\bigl(L_h\varepsilon+Ch^q\bigr)\le\varepsilon
$
and $r_0\le\varepsilon$. Then the corrected scheme satisfies
\[
r_j\le\varepsilon
\qquad\text{for all } j\ge0.
\]
In particular, the bound is independent of the number of time steps.
\end{theorem}

Classical timestep restrictions such as the CFL condition are not imposed directly here. However, the timestep $h$ remains implicitly restricted by the requirement that the predictor state $\widehat{u}_{j+1}$ lie within the valid basin of attraction of the score correction. Full details and proofs are provided in Appendix A.

\begin{figure}[t]
    \centering
    \includegraphics[width=\linewidth]{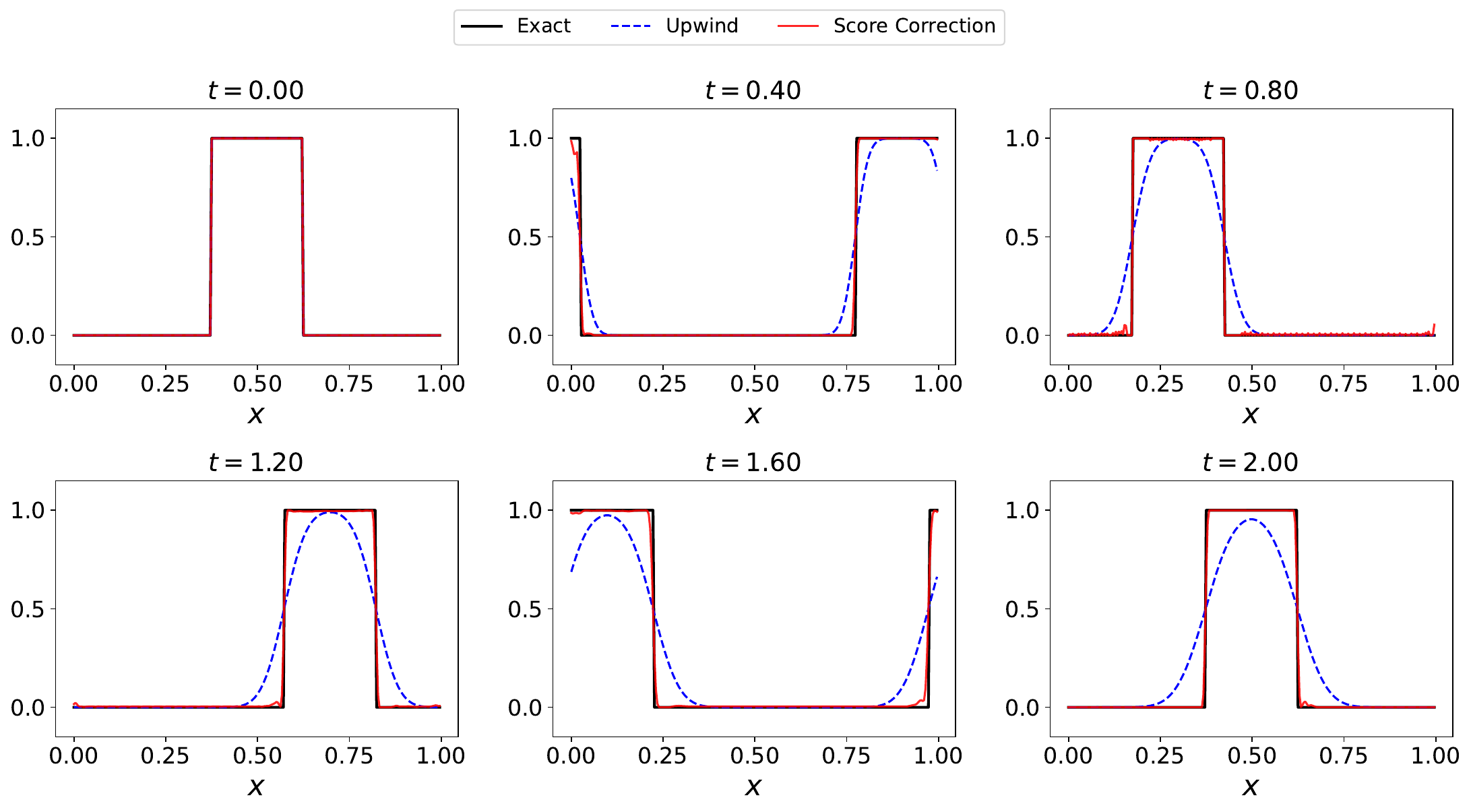}
    \caption{Box-pulse advection. The score-corrected method (red) maintains the plateau.}
    \label{fig:transport_box_snapshots}
\end{figure}

\section{Experiments}
\label{sec:experiments}

We test our method over several well-known PDE instances, specifically 1D Advection \citep{leveque1992numerical}, Korteweg–de Vries (KdV) \citep{Ascher2005Symplectic}, Nonlinear Schrödinger (NLS) \citep{sulem2007nonlinear}, and Burgers' equations \citep{burgers1948mathematical}, where for the latter we use the setting provided by PDEBench \citep{PDEBench2022}. The core objective of these experiments is to demonstrate that in settings where standard numerical methods fail or diverge, applying our score-based correction successfully stabilizes the dynamics and enables accurate long-time integration.

We evaluate our approach against uncorrected baselines and standard stabilization schemes: TVD \citep{harten1997high}, WENO-5 \citep{jiang1996efficient}, and Spectral Filtering \citep{gottlieb2001spectral}. While these classical methods effectively prevent catastrophic instabilities, they frequently introduce significant numerical diffusion that over-smooths sharp profiles or remain highly specialized to specific physical regimes. We note that some classical stabilization methods are not universally applicable
in all problem settings.

Across all experiments, the learned score model $s_\theta$ is parameterized as a convolutional U-Net \citep{ronneberger2015u}.
During the offline phase, the model is trained across several various noise levels. Technical details, additional implementation details of the baseline methods, and ablation study are provided in the appendices. 

\begin{figure}[t]
    \centering
    \includegraphics[width=\linewidth]{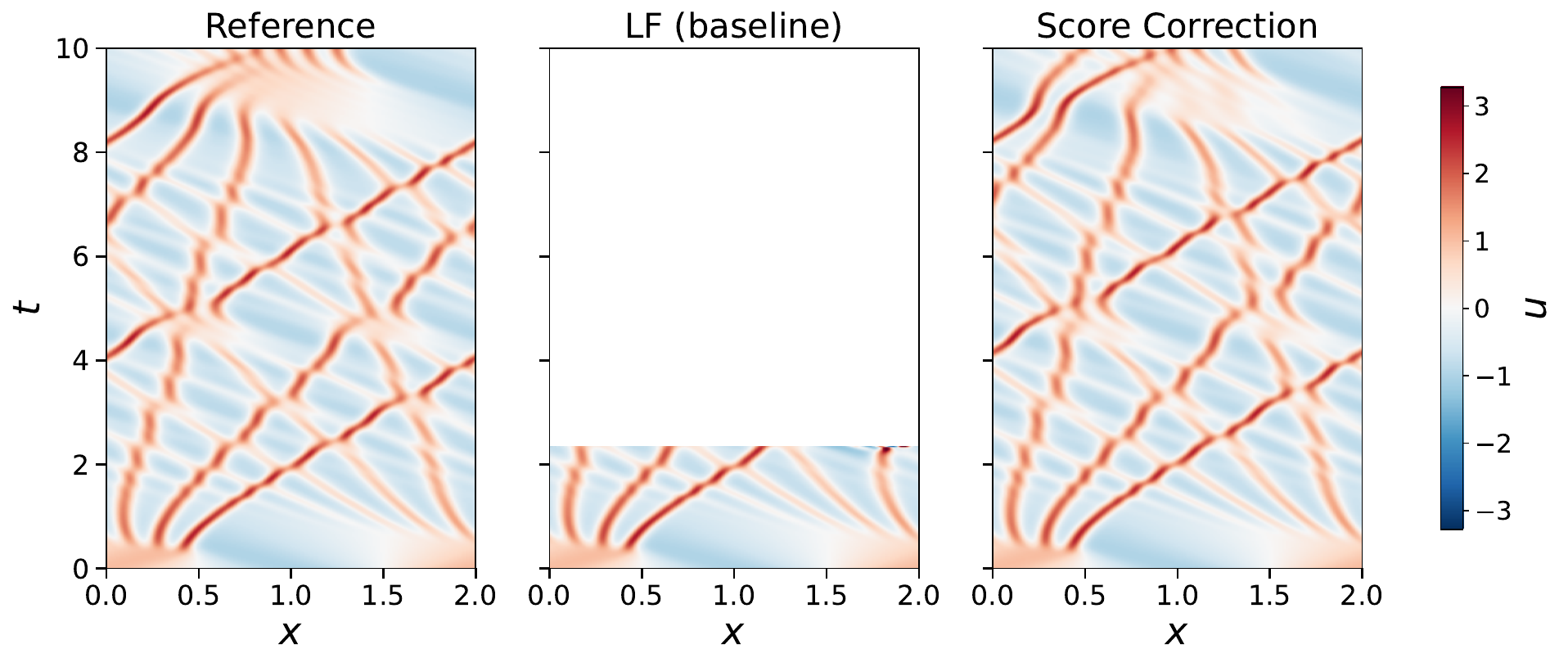}
    \caption{KdV results using
    $\Delta t = 5\!\times\!10^{-4}$. \emph{Left to right:}
    reference; leap-frog baseline (blows up at $t \approx 2.34$); Score
    Correction. The Score Correction reproduces the
    diagonal soliton trajectories of the reference.}
    \label{fig:kdv_spacetime_dt5e-4}
\end{figure}

\subsection{1D Advection}
\label{subsec:phase1_advection}

We consider the 1D advection equation 
\begin{equation}
    u_t + a u_x = 0, \qquad x \in [0,L], \quad t \in [0,T],
\end{equation}
with constant advection speed $a > 0$ and periodic boundary conditions.
We discretize on a uniform grid with $N = 256$ cells over $[0,1]$ and choose the first-order upwind scheme as the baseline method $G_{\Delta t}$. The provisional update reads
\begin{equation}
    \hat{u}^{\,n+1}_j = u^n_j - \nu\bigl(u^n_j - u^n_{j-1}\bigr), \qquad \nu = {a\,\Delta t}\cdot{(\Delta x)^{-1}}.
\end{equation}
The upwind scheme is stable for small values of \(\nu \leq 1\), but numerical instability is introduced  for $\nu > 1$.

Consider the following box pulse initial condition over the domain $[0,1]$ at final time $T = 2$ (two full domain traversals):

\begin{equation}
    u_0(x) = \mathbf{1}_{[x_1,\, x_1 + w]}(x).
\end{equation}
This discontinuous profile tests whether the correction sharpens diffused fronts without generating spurious oscillations. Results are reported in Figure~\ref{fig:transport_box_snapshots}, where we applied the score correction using step size magnitude of $0.03$.
The score-based correction effectively mitigates the numerical errors of the baseline upwind scheme. For the smooth Gaussian profile, the corrected method preserves the peak amplitude throughout the integration. In the box scenarios, the correction maintains the flat plateaus of the single box pulse and preserves both the amplitude and the gap separation in the two-box configuration. 

\begin{table*}[t]
    \centering
    \footnotesize
\caption{Performance of score-based corrections for the KdV equation. The baseline LF scheme becomes unstable and blows up at $t = 2.34$. Classical stabilizers avoid blow-up but sacrifice structure: WENO-5 over-diffuses the solitons, while spectral filtering badly violates the Hamiltonian. Only the score correction stabilizes the integration up to $t = 10$ while conserving the invariants.}

    \label{tab:kdv_marginal}
    \begin{tabular}{ l c | c c c | c c c}
        \hline
         \multirow{2}{*}{method} & \multirow{2}{*}{blew up} & \multicolumn{3}{c}{at $t = 2$} & \multicolumn{3}{c}{at $t = 10$} \\
        \cline{3-5} \cline{6-8}
          & & $\mathrm{err}_{I_2}$ & $\mathrm{err}_H$ & $\mathrm{err}_{\mathrm{pt}}$ & $\mathrm{err}_{I_2}$ & $\mathrm{err}_H$ & $\mathrm{err}_{\mathrm{pt}}$ \\
        \hline
         LF (baseline)                  & 2.34   & $4.30\times10^{-2}$ & $7.85\times10^{-1}$ & 0.43 &       & \textit{div.}      &   \\
         WENO-5             & --     & $2.12\times10^{-1}$  & $6.12$              & 0.78 & $5.03\times10^{-1}$  & $8.22$               & 0.91 \\
         Spectral Filtering  & --     & $1.28\times10^{-6}$  & $2.54$              & 0.10 & $5.95\times10^{-6}$  & $2.74$               & 0.09 \\
Score Correction  & -- & 0        & 1.42$\times10^{-8}$ & 0.42 & 0        & 4.52$\times10^{-8}$  & 0.70 \\
        \hline
    \end{tabular}
\end{table*}

\begin{table*}[t]
    \centering
    \footnotesize
    \caption{NLS results using $\Delta x = 0.1$. The uncorrected baseline suffers from severe amplitude distortion and unbounded Hamiltonian error growth. Spectral filtering keeps the energy error bounded by leaving a residual drift, and collapse the wave amplitude at $\Delta t = 0.5$. Only the score-corrected scheme preserves the wave's structural integrity while restricting energy drift at both timesteps.}
    \label{tab:nls_marginal}
    \begin{tabular}{c l | c c | c c | c c}
    \hline
    \multirow{2}{*}{$\Delta t$} & \multirow{2}{*}{method} & \multicolumn{2}{c|}{$t = 100$} & \multicolumn{2}{c|}{$t = 200$} & \multicolumn{2}{c}{$t = 400$} \\
    \cline{3-8}
     & & \(\max_x|\psi|\) & err($H$) & \(\max_x|\psi|\) & err($H$) & \(\max_x|\psi|\) & err($H$) \\
    \hline
    \multirow{3}{*}{0.20} & Baseline & 1.061 & $4.899\!\times\!10^{-1}$ & 1.104 & $5.518$ & 1.049 & $1.788\!\times\!10^{1}$ \\
  & Spectral Filtering & 1.061 & $1.068\!\times\!10^{-1}$ & 1.105 & $1.842\!\times\!10^{-1}$ & 1.040 & $1.983$ \\
  & Score Correction     & 0.979 & $6.174\!\times\!10^{-2}$ & 0.983 & $2.199\!\times\!10^{-2}$ & 0.985 & $1.618\!\times\!10^{-1}$ \\
    \hline
    \multirow{3}{*}{0.50} & Baseline & 0.850 & $4.596\!\times\!10^{2}$ & 0.651 & $3.909\!\times\!10^{2}$ & 0.641 & $4.899\!\times\!10^{2}$ \\
  & Spectral Filtering & 0.362 & $1.077\!\times\!10^{2}$ & 0.378 & $9.244\!\times\!10^{1}$ & 0.375 & $8.708\!\times\!10^{1}$ \\
  & Score Correction     & 0.971 & $3.07\!\times\!10^{-2}$ & 0.965 & $9.428\!\times\!10^{-2}$ & 0.967 & $1.984\!\times\!10^{-2}$ \\
    \hline
    \end{tabular}
\end{table*}

\subsection{Korteweg--de Vries (KdV)}\label{sec:kdv}
Consider the nonlinear PDE:
\begin{equation}\label{eq:kdv}
    u_t \;=\; \beta\,(u^2)_x \,+\, \nu\, u_{xxx},
    \qquad x \in [0, L],\; t \in [0, T],
\end{equation}
with periodic boundary conditions and coefficients
$\beta = -0.5$, $\nu = -(0.022)^2$.
We fix $L = 2$, $T = 10$, and initial condition
\begin{equation}\label{eq:kdv_ic}
    u_0(x) = \cos(\pi x).
\end{equation}
The equation admits three quantities that are invariant in time. It preserves the mass and the $L^2$ norm, and the Hamiltonian:
\begin{equation}\label{eq:kdv_invariants}
  H(u) = -\tfrac{1}{6}\!\int_0^L u^3\,dx + \tfrac{-\nu}{2}\!\int_0^L u_x^2\,dx.
\end{equation}


On the inference grid with
$N = 200$ cells over $[0,L]$ ($\Delta x = L/N = 0.01$) and time-step $\Delta t$, we utilize the conservative three-level
method by \citet{zakr}
\begin{equation}\label{eq:zk_lf}
\begin{split}
  u_j^{n+1} &= u_j^{n-1} + \tfrac{2\beta\,\Delta t}{3\,\Delta x}\bigl(u_{j-1}^n + u_j^n + u_{j+1}^n\bigr)\bigl(u_{j+1}^n - u_{j-1}^n\bigr) \\
  &\quad + \tfrac{\nu\,\Delta t}{\Delta x^3}\bigl(u_{j+2}^n - 2u_{j+1}^n + 2u_{j-1}^n - u_{j-2}^n\bigr).
\end{split}
\end{equation}

For \(\hat{u}^{\,n+1}\), we denote the leap-frog provisional update~\Eqref{eq:zk_lf}
and $s_\theta(u)$ the learned score
of Section~\ref{sec:score_corrections}.
For the KdV equation, the stabilization operator $\mathcal{S}$ is instantiated to incorporate both score-based correction and the enforcement of the conserved quantities. The resulting update combines a learned score correction with projections that enforce the invariants of the system.


Training samples are produced by
integrating \Eqref{eq:zk_lf} on a finer spatial grid
\(N_{\mathrm{fine}}\!=\!400\) ($\Delta x_{\mathrm{fine}}\!=\!0.005$) at
\(\Delta t_{\mathrm{fine}}\!=\!6\!\times\!10^{-5}\) from several
initial conditions, then spatially downsampled to the inference grid. We report the norm and Hamiltonian errors, along with the point-wise error,
obtained using a correction step size of $3{\times}10^{-4}$.

%

The primary objective of the score-based correction is to maintain stability and structure where classical methods fail. As detailed in Table~\ref{tab:kdv_marginal}, the baseline leap-frog (LF) scheme exhibits numerical blow-up.
To avoid phase drifting, we additionally feed the model the previous state during training.

Applying the score correction successfully stabilizes the integration up to $t = 10$, driving the $I_2$ error to machine precision and the Hamiltonian error ($\mathrm{err}_H$) to roughly $10^{-9}$. As shown in Table~\ref{tab:kdv_marginal}, it keeps pointwise errors low by ensuring the solitons do not drift from their true trajectories. Visually, the spacetime plot in Figure~\ref{fig:kdv_spacetime_dt5e-4} confirms this, with the diagonal soliton trajectories of the corrected scheme accurately tracking the slopes of the reference.

To understand this behavior, consider the stability condition obtained by
freezing the coefficients in \Eqref{eq:kdv},
\begin{equation}
\Delta t < {\Delta x}{(\nu/\Delta x^2 + 2|\beta u_{\max}|)^{-1}}.
\end{equation}
For the exact soliton solution, $|u_{\max}|$ is constant, so the condition
remains satisfied if it holds initially. In contrast, the uncorrected
numerical solution exhibits growth in $|u_{\max}|$, eventually violating the
stability bound and leading to blowup even if for a very small time-step size. The score correction prevents this
growth and maintains an approximately constant solution amplitude, thereby
preserving stability over long-time integration.

\begin{table*}[t]
    \centering
    \footnotesize
    \caption{1D Burgers on PDEBench. Mean absolute ($\mathrm{err}_{L^1}$), relative root-mean-square ($\mathrm{err}_{L^2}$), and maximum pointwise ($\mathrm{err}_{L^\infty}$) errors across the held-out test ensemble at $t=0.05$ and $t=0.5$. The uncorrected baseline, and the compared methods diverge (\textit{div.}), while only the score correction stabilizes the integration across all timesteps. The \emph{blew up} column reports the mean divergence time, and $N$ is the number of steps to $t=0.5$.}

    \label{tab:burgers_stability}
    \setlength{\tabcolsep}{4pt}
    \begin{tabular}{c c l | c | c c c | c c c}
        \hline

\multirow{2}{*}{$\Delta t$} & \multirow{2}{*}{$N$} & \multirow{2}{*}{method} & \multirow{2}{*}{blew up} & \multicolumn{3}{c|}{at $t = 0.05$} & \multicolumn{3}{c}{at $t = 0.5$} \\
\cline{5-7}\cline{8-10}
 & & & & $\mathrm{err}_{L^1}$ & $\mathrm{err}_{L^2}$ & $\mathrm{err}_{L^\infty}$ & $\mathrm{err}_{L^1}$ & $\mathrm{err}_{L^2}$ & $\mathrm{err}_{L^\infty}$ \\
 \hline
\multirow{5}{*}{$3.33\!\times\!10^{-3}$} & \multirow{5}{*}{150}
 & baseline & 0.066 & 0.205 & 1.231 & 7.115 & \multicolumn{3}{c}{\textit{div.}} \\
 & & TVD (Superbee) & 0.018 & \multicolumn{3}{c|}{\textit{div.}} & \multicolumn{3}{c}{\textit{div.}} \\
 & & WENO-5 & 0.021 & \multicolumn{3}{c|}{\textit{div.}} & \multicolumn{3}{c}{\textit{div.}} \\
 & & Spectral Filtering & 0.151 & 0.049 & 0.117 & 0.974 & \multicolumn{3}{c}{\textit{div.}} \\
 & & Score Correction & -- & 0.016 & 0.037 & 0.114 & 0.041 & 0.162 & 0.211 \\
 \hline
\multirow{5}{*}{$4.00\!\times\!10^{-3}$} & \multirow{5}{*}{125}
 & baseline & 0.056 & 0.319 & 2.688 & 18.55 & \multicolumn{3}{c}{\textit{div.}} \\
 & & TVD (Superbee) & 0.019 & \multicolumn{3}{c|}{\textit{div.}} & \multicolumn{3}{c}{\textit{div.}} \\
 & & WENO-5 & 0.019 & \multicolumn{3}{c|}{\textit{div.}} & \multicolumn{3}{c}{\textit{div.}} \\
 & & Spectral Filtering & 0.101 & 0.106 & 0.237 & 1.859 & \multicolumn{3}{c}{\textit{div.}} \\
 & & Score Correction & -- & 0.024 & 0.071 & 0.435 & 0.028 & 0.174 & 0.157 \\
\hline
    \end{tabular}
\end{table*}

\subsection{Nonlinear Schr\"odinger}\label{sec:nls}

Next, consider the cubic nonlinear Schr\"odinger (NLS)
equation in one space variable
\begin{equation}\label{eq:nls}
    \psi_t \;=\; i\bigl(\psi_{xx} + |\psi|^2 \psi\bigr),
    \qquad x \in [-20,\,80],\; t \in [0, T],
\end{equation}
with periodic boundary conditions and the
initial condition
\begin{equation}\label{eq:nls_ic}
    \psi_0(x) \;=\; e^{ix/2}\,\mathrm{sech}\!\left(\tfrac{x}{\sqrt 2}\right)
                  + e^{i(x-25)/20}\,\mathrm{sech}\!\left(\tfrac{x-25}{\sqrt 2}\right).
\end{equation}
The continuum dynamics conserves the norm
and the integrated Hamiltonian,
\begin{equation}\label{eq:nls_invariants}
    H(\psi, \bar\psi) \;=\; \psi_x\,\bar\psi_x \;-\; \tfrac{1}{2}\,\psi^2\,\bar\psi^2.
\end{equation}

We use a popular splitting method for \Eqref{eq:nls}. Given the
state $\psi_j^{n-1}$ at the previous step, one timestep advances in two
stages.

\textbf{(1) Linear PDE step.} For the constant-coefficient PDE $\psi_t = i\,\psi_{xx}$, apply standard 3-point
centered spatial discretization and implicit midpoint in time: setting
$v_j^{n-1} := \psi_j^{n-1}$, solve for $v_j^n$
\begin{equation}\label{eq:nls_lin}
\begin{split}
  \frac{v_j^n - v_j^{n-1}}{\Delta t} &= \frac{i}{2\,\Delta x^2} \Bigl[\bigl(v_{j+1}^n - 2\,v_j^n + v_{j-1}^n\bigr) \\
  &\quad + \bigl(v_{j+1}^{n-1} - 2\,v_j^{n-1} + v_{j-1}^{n-1}\bigr)\Bigr].
\end{split}
\end{equation}

\textbf{(2) Nonlinear ODE step.} For the other half of the PDE, $\psi_t = i\,| \psi |^2\,\psi$,
the exact pointwise solution is the phase rotation
\begin{equation}\label{eq:nls_nl}
    \psi_j^n \;=\; v_j^n\,\exp\!\bigl(i\,\Delta t\,|v_j^n|^2\bigr).
\end{equation}

To ensure stability of this symplectic splitting scheme, the time step must be appropriately limited, even though the implicit midpoint method is unconditionally stable in time \citep{faou12,wehe,lubich08}. We show next that our score-based scheme allows usage of larger time steps with clean results.

Training samples are produced by integrating the splitting
scheme using \Eqref{eq:nls_lin} and \Eqref{eq:nls_nl} on a fine grid
$(\Delta t_{\mathrm{fine}}, \Delta x_{\mathrm{fine}}) = (10^{-3},\,0.01)$
from the initial condition~\Eqref{eq:nls_ic}, then
block-averaged to the inference grid, 
and a correction step size $\eta = 1{\times}10^{-3}$.
Table~\ref{tab:nls_marginal} reports the relative errors for
the baseline and the score-corrected scheme, and we additionally provide a visualization of out method in the Appendix.



\subsection{Burgers' Equation}\label{sec:burgers_pdebench}
Consider the following PDE:
\begin{equation}\label{eq:burgers_pde}
    u_t \;+\; u\,u_x \;=\; \nu\,u_{xx},
    \qquad x \in [0, L],\ t \in [0, T],
\end{equation}
with periodic boundary conditions, $L = 1$, $T = 2$, and viscosity \(\nu > 0\).
We follow PDEBench \citep{PDEBench2022} formulation of this problem, and employ a correction step size of $0.01$,
and use as baseline integrator $G_{\Delta t}$ the well known forward-time centered-space
(FTCS) scheme on $N = 256$ cells. 

The training samples are PDEBench fine-grid Burgers snapshots on a $N=1024$ points grid, spatially downsampled by uniform striding to the
inference resolution $N = 256$. We report the
mean $L^1$, relative
$L^2$, and $L^\infty$ errors against the reference, evaluated at two times:
an early time $t = 0.05$, before baseline blow-up, and the active-window end $t = 0.5$.

Table~\ref{tab:burgers_stability} reports the score correction results and
in the Appendix we present   the corresponding space-time evolution. We observe that using score correction enables the stabilization of the baseline numerical method.

\section{Conclusion}

We have presented a score-based stabilization framework that addresses the fundamental trade-off between stability and computational efficiency in time-stepping schemes for PDEs. By reformulating stabilization as a two-stage procedure and applying a learned score correction after provisional numerical updates, we mitigate nonphysical behaviors and numerical instabilities that can arise in standard integrators.
We further show that the resulting scheme is stable, enabling robust long-time integration without restrictive timestep constraints. This stability property is a direct consequence of the contraction behavior of the proposed stabilization operator toward the set of physically admissible states.


We evaluated our method on a range of benchmark PDEs, including the 1D advection equation, the Korteweg-de Vries (KdV) equation, the nonlinear Schrödinger (NLS) equation, and Burgers' equation. Across our experiments, the proposed stabilization consistently improves robustness and preserves the qualitative structure of the underlying dynamics, avoiding the consistent solution-smoothing artifacts such as artificial diffusion that typically compromise classical stabilization schemes. This work opens the door to using score-based corrections as a general stabilization layer for marginally stable numerical simulation.

\bibliography{biblio}

@article{Ascher2004Multisymplectic,
  author  = {Ascher, Uri M. and McLachlan, Robert I.},
  title   = {Multisymplectic box schemes and the {K}orteweg--de {V}ries equation},
  journal = {Applied Numerical Mathematics},
  volume  = {48},
  number  = {3--4},
  pages   = {255--269},
  year    = {2004},
  doi     = {10.1016/S0168-9274(03)00154-5}
}

@article{Ascher2005Symplectic,
  author  = {Ascher, Uri M. and McLachlan, Robert I.},
  title   = {On symplectic and multisymplectic schemes for the {K}d{V} equation},
  journal = {Journal of Scientific Computing},
  volume  = {25},
  number  = {1--2},
  pages   = {83--104},
  year    = {2005},
  doi     = {10.1007/s10915-004-4634-6}
}

@inproceedings{Song2021ScoreBasedSDE,
  author    = {Song, Yang and Sohl-Dickstein, Jascha and Kingma, Diederik P. and Kumar, Abhishek and Ermon, Stefano and Poole, Ben},
  title     = {Score-Based Generative Modeling through Stochastic Differential Equations},
  booktitle = {International Conference on Learning Representations},
  year      = {2021}
}

@article{Hyvarinen2005ScoreMatching,
  author  = {Hyv{\"a}rinen, Aapo},
  title   = {Estimation of Non-Normalized Statistical Models by Score Matching},
  journal = {Journal of Machine Learning Research},
  volume  = {6},
  pages   = {695--709},
  year    = {2005}
}

@article{Efron2011Tweedie,
  author  = {Efron, Bradley},
  title   = {Tweedie's Formula and Selection Bias},
  journal = {Journal of the American Statistical Association},
  volume  = {106},
  number  = {496},
  pages   = {1602--1614},
  year    = {2011},
  doi     = {10.1198/jasa.2011.tm11181}
}

@Book{faou12,
  Title = {Geometric Numerical Integration and Schr\"odinger Equations},
  Author = {Erwan Faou},
  Publisher = {European Mathematical Society},
  note = {Zurich Lectures in Advanced Mathematics},
  Year= {2012}
}

@inproceedings{PDEBench2022,
author = {Takamoto, Makoto and Praditia, Timothy and Leiteritz, Raphael and MacKinlay, Dan and Alesiani, Francesco and Pflüger, Dirk and Niepert, Mathias},
title = {{PDEBench: An Extensive Benchmark for Scientific Machine Learning}},
year = {2022},
booktitle = {36th Conference on Neural Information Processing Systems (NeurIPS 2022) Track on Datasets and Benchmarks},
url = {https://arxiv.org/abs/2210.07182}
}

@book{sulem2007nonlinear,
  title={The nonlinear Schr{\"o}dinger equation: self-focusing and wave collapse},
  author={Sulem, Catherine and Sulem, Pierre-Louis},
  volume={139},
  year={2007},
  publisher={Springer Science \& Business Media}
}

@article{burgers1948mathematical,
  title={A mathematical model illustrating the theory of turbulence},
  author={Burgers, Johannes Martinus},
  journal={Advances in applied mechanics},
  volume={1},
  pages={171--199},
  year={1948},
  publisher={Elsevier}
}

@ARTICLE{wehe,
       AUTHOR= {J.A.C. Weideman and B.M. Herbst},
       TITLE = {Split-step methods for the solution of the nonlinear
{S}chr\"odinger equation},
       journal = {SIAM J. Numer. Anal.},
       volume = {23},
       pages = {485-507},
       year = {1986}
}

@ARTICLE{lubich08,
       AUTHOR= {C. Lubich},
       TITLE = {On splitting methods for {S}chr\"odinger-{P}oisson and
cubic nonlinear {S}chr\"odinger equations},
       journal = {Math. Comp.},
       volume = {77},
       pages = {2141-2153},
       year = {2008}
}

@inproceedings{luo2021score,
  title={Score-based point cloud denoising},
  author={Luo, Shitong and Hu, Wei},
  booktitle={Proceedings of the IEEE/CVF international conference on computer vision},
  pages={4583--4592},
  year={2021}
}

@inproceedings{ronneberger2015u,
  title={U-net: Convolutional networks for biomedical image segmentation},
  author={Ronneberger, Olaf and Fischer, Philipp and Brox, Thomas},
  booktitle={International Conference on Medical image computing and computer-assisted intervention},
  pages={234--241},
  year={2015},
  organization={Springer}
}

@book{hlw,
  author =      "E. Hairer and C. Lubich and G. Wanner",
  title =       "Geometric Numerical Integration",
  publisher =   "Springer",
  year =        "2002"
  }

@article{gottlieb2001spectral,
  title={Spectral methods for hyperbolic problems},
  author={Gottlieb, David and Hesthaven, Jan S},
  journal={Journal of Computational and Applied Mathematics},
  volume={128},
  number={1-2},
  pages={83--131},
  year={2001},
  publisher={Elsevier}
}

@article{zakr,
         author = {N. J. Zabusky and M. D. Kruskal},
         title = {Interaction of 'solitons' in a collisionless
                  plasma and the recurrence of initial states},
  journal={Physical review letters},
         volume = {15},
         year = {1965},
         pages = {240-243}
}

@article{kingma2014adam,
  title={Adam: A method for stochastic optimization},
  author={Kingma, Diederik P and Ba, Jimmy},
  journal={arXiv preprint arXiv:1412.6980},
  year={2014}
}

@book{leveque2002finite,
  title={Finite volume methods for hyperbolic problems},
  author={LeVeque, Randall J},
  volume={31},
  year={2002},
  publisher={Cambridge university press}
}

@article{harten1997high,
  title={High resolution schemes for hyperbolic conservation laws},
  author={Harten, Ami},
  journal={Journal of computational physics},
  volume={135},
  number={2},
  pages={260--278},
  year={1997},
  publisher={Elsevier}
}

@article{godunov1959difference,
  title={A difference scheme for numerical solution of discontinuous solution of hydrodynamic equations},
  author={Godunov, Sergei Konstantinovich},
  journal={Math. Sbornik},
  volume={47},
  pages={271--306},
  year={1959}
}

@article{raissi2019physics,
  title={Physics-informed neural networks: A deep learning framework for solving forward and inverse problems involving nonlinear partial differential equations},
  author={Raissi, Maziar and Perdikaris, Paris and Karniadakis, George E},
  journal={Journal of Computational physics},
  volume={378},
  pages={686--707},
  year={2019},
  publisher={Elsevier}
}

@article{miyatake2025quantifying,
  title={Quantifying uncertainty in the numerical integration of evolution equations based on Bayesian isotonic regression},
  author={Miyatake, Yuto and Irie, Kaoru and Matsuda, Takeru},
  journal={Japan Journal of Industrial and Applied Mathematics},
  volume={42},
  number={3},
  pages={983--1001},
  year={2025},
  publisher={Springer}
}

@article{maccormack2003effect,
  title={The effect of viscosity in hypervelocity impact cratering},
  author={MacCormack, Robert W},
  journal={Journal of spacecraft and rockets},
  volume={40},
  number={5},
  pages={757--763},
  year={2003}
}

@article{toyota2025joint,
  title={Joint Bayesian Inference of Parameter and Discretization Error Uncertainties in ODE Models},
  author={Toyota, Shoji and Miyatake, Yuto},
  journal={arXiv preprint arXiv:2511.23010},
  year={2025}
}

@article{matsuda2021estimation,
  title={Estimation of ordinary differential equation models with discretization error quantification},
  author={Matsuda, Takeru and Miyatake, Yuto},
  journal={SIAM/ASA Journal on Uncertainty Quantification},
  volume={9},
  number={1},
  pages={302--331},
  year={2021},
  publisher={SIAM}
}

@article{de2022riemannian,
  title={Riemannian score-based generative modelling},
  author={De Bortoli, Valentin and Mathieu, Emile and Hutchinson, Michael and Thornton, James and Teh, Yee Whye and Doucet, Arnaud},
  journal={Advances in neural information processing systems},
  volume={35},
  pages={2406--2422},
  year={2022}
}

@article{kochkov2021machine,
  title={Machine learning--accelerated computational fluid dynamics},
  author={Kochkov, Dmitrii and Smith, Jamie A and Alieva, Ayya and Wang, Qing and Brenner, Michael P and Hoyer, Stephan},
  journal={Proceedings of the National Academy of Sciences},
  volume={118},
  number={21},
  pages={e2101784118},
  year={2021},
  publisher={National Academy of Sciences}
}

@article{bar2019learning,
  title={Learning data-driven discretizations for partial differential equations},
  author={Bar-Sinai, Yohai and Hoyer, Stephan and Hickey, Jason and Brenner, Michael P},
  journal={Proceedings of the National Academy of Sciences},
  volume={116},
  number={31},
  pages={15344--15349},
  year={2019},
  publisher={National Academy of Sciences}
}

@article{meng2021sdedit,
  title={Sdedit: Guided image synthesis and editing with stochastic differential equations},
  author={Meng, Chenlin and He, Yutong and Song, Yang and Song, Jiaming and Wu, Jiajun and Zhu, Jun-Yan and Ermon, Stefano},
  journal={arXiv preprint arXiv:2108.01073},
  year={2021}
}

@article{jiang1996efficient,
  title={Efficient implementation of weighted ENO schemes},
  author={Jiang, Guang-Shan and Shu, Chi-Wang},
  journal={Journal of computational physics},
  volume={126},
  number={1},
  pages={202--228},
  year={1996},
  publisher={Elsevier}
}

@article{vonneumann1950method,
  title={A method for the numerical calculation of hydrodynamic shocks},
  author={VonNeumann, John and Richtmyer, Robert D},
  journal={Journal of applied physics},
  volume={21},
  number={3},
  pages={232--237},
  year={1950},
  publisher={American Institute of Physics}
}

@article{boris1973flux,
  title={Flux-corrected transport. I. SHASTA, a fluid transport algorithm that works},
  author={Boris, Jay P and Book, David L},
  journal={Journal of computational physics},
  volume={11},
  number={1},
  pages={38--69},
  year={1973},
  publisher={Elsevier}
}

@book{lax1973hyperbolic,
  title={Hyperbolic systems of conservation laws and the mathematical theory of shock waves},
  author={Lax, Peter D},
  year={1973},
  publisher={SIAM}
}

@book{leveque1992numerical,
	Author = {R.J. LeVeque},
	Publisher = {Birkhauser},
	Title = {Numerical Methods for Conservation Laws},
	Year = {1990}}

@article{tb,
	Author = {P. Tanarkit and L. Biegler},
	Journal = {Ind. Eng. Chem. Res.},
	Pages = {1253--1266},
	Title = {Stable decomposition for dynamic optimization},
	Volume = {34},
	Year = {1995}}

@book{rm,
	Author = {R.D. Richtmyer and K.W. Morton},
	Publisher = {Wiley},
	Title = {Difference Methods for Initial-Value Problems},
	Year = {1967}}

@book{pt,
	Author = {Peyret, R. and T. Taylor},
	Publisher = {Spring er-Verlag},
	Title = {Computational Methods for Fluid Flow},
	Year = {1983}}

@article{ho2020denoising,
  title={Denoising diffusion probabilistic models},
  author={Ho, Jonathan and Jain, Ajay and Abbeel, Pieter},
  journal={Advances in neural information processing systems},
  volume={33},
  pages={6840--6851},
  year={2020}
}


\onecolumn
\appendix

\setcounter{theorem}{0}
\setcounter{lemma}{0}

\renewcommand{\thetheorem}{\Alph{section}.\arabic{theorem}}
\renewcommand{\thelemma}{\Alph{section}.\arabic{lemma}}

\section{Proofs of Theoretical Results}
We provide the complete proofs for the results presented in \Cref{sec:math_analysis}.
Let $X$ be a Hilbert space and let $\mathcal{M}\subset X$ denote the
manifold of physically admissible states.

\begin{lemma}[Score correction preserves order of accuracy]
Let $u(t_{j+1})\in\mathcal{M}$ be the exact solution and assume the provisional
update satisfies
\[
\|\widehat{u}_{j+1}-u(t_{j+1})\|_X \le Ch^p .
\]
Let the correction operator be
\[
R(u) = u + \eta s(u),
\]
or more generally a finite composition of such maps. Assume that $R$ is
Lipschitz-continuous on a neighborhood of $\mathcal{M}$, with Lipschitz
constant $L_R$ independent of $h$, and satisfies
\[
R(u)=u
\qquad \text{for all } u\in\mathcal{M}.
\]
Define $u_{j+1}=R(\widehat{u}_{j+1})$. Then, for sufficiently small $h$,
\[
\|u_{j+1}-u(t_{j+1})\|_X \le L_R C h^p .
\]
Thus, the correction does not reduce the order of accuracy.
\end{lemma}

\begin{proof}
Since $R(u(t_{j+1}))=u(t_{j+1})$, we have
\begin{equation}
\begin{split}
\|u_{j+1}-u(t_{j+1})\|_X
&=
\|R(\widehat{u}_{j+1})-R(u(t_{j+1}))\|_X \\
&\le
L_R\|\widehat{u}_{j+1}-u(t_{j+1})\|_X \\
&\le
L_R C h^p .
\end{split}
\end{equation}
\end{proof}

\begin{lemma}[One-step distance bound under the numerical step]
Let $G_h$ be a one-step numerical method satisfying
\[
\|G_h(u)-\Phi_h(u)\|_X \le Ch^{q},
\qquad u\in\mathcal{M},
\]
where $\Phi_h$ is the exact flow. Further, assume that $G_h$ is
Lipschitz-continuous on a neighborhood of $\mathcal{M}$:
\[
\|G_h(x)-G_h(y)\|_X \le L_h \|x-y\|_X.
\]
Assume also that $\Phi_h$ leaves $\mathcal{M}$ invariant, that is,
\[
\Phi_h(\mathcal{M})\subseteq\mathcal{M}.
\]
Then, for any $x\in X$ with
$\operatorname{dist}_X(x,\mathcal{M})\le \varepsilon$, where the
$\varepsilon$-neighborhood of $\mathcal{M}$ lies in the neighborhood
on which the Lipschitz estimate holds,
\[
\operatorname{dist}_X(G_h(x),\mathcal{M})
\le L_h\varepsilon+Ch^{q}.
\]
\end{lemma}

\begin{proof}
For any $\delta>0$, choose $u\in\mathcal{M}$ such that
\[
\|x-u\|_X
\le \operatorname{dist}_X(x,\mathcal{M})+\delta
\le \varepsilon+\delta.
\]
Since $\Phi_h(u)\in\mathcal{M}$, we have
\[
\operatorname{dist}_X(G_h(x),\mathcal{M})
\le \|G_h(x)-\Phi_h(u)\|_X.
\]
Furthermore,
\begin{equation}
\begin{split}
\|G_h(x)-\Phi_h(u)\|_X
&\le
\|G_h(x)-G_h(u)\|_X
+\|G_h(u)-\Phi_h(u)\|_X \\
&\le
L_h\|x-u\|_X+Ch^q \\
&\le
L_h(\varepsilon+\delta)+Ch^q.
\end{split}
\end{equation}
Letting $\delta\to0$ completes the proof.
\end{proof}

\begin{lemma}[Score-based stabilization contracts toward a manifold]
Assume there exists an energy functional
$E:X\to\mathbb{R}_{\ge 0}$ such that
\[
E(u)=0 \iff u\in\mathcal{M},
\qquad
\nabla E(u)=0
\quad \text{for all } u\in\mathcal{M}.
\]
Assume further that, on the sublevel set
\[
\mathcal{N}=\{u\in X:E(u)\le E_{\max}\},
\]
there exist constants $c_-,c_+>0$ such that
\[
c_-\,\operatorname{dist}_X(u,\mathcal{M})^2
\le E(u)
\le
c_+\,\operatorname{dist}_X(u,\mathcal{M})^2.
\]

Assume the stabilization operator is gradient-based:
\[
R=(\operatorname{Id}-\eta\nabla E)^{\circ K},
\]
with $0<\eta\le 1/L$, where $E$ has an
$L$-Lipschitz-continuous gradient and satisfies the PL inequality
\[
\|\nabla E(u)\|_X^2 \ge 2\mu E(u)
\]
on $\mathcal{N}$. Then, for all $u\in\mathcal{N}$,
\[
E(R(u))\le (1-\mu\eta)^K E(u).
\]
Consequently,
\[
\operatorname{dist}_X(R(u),\mathcal{M})
\le
\sqrt{\frac{c_+}{c_-}}\,
(1-\mu\eta)^{K/2}
\operatorname{dist}_X(u,\mathcal{M}).
\]
\end{lemma}

\begin{proof}
Let
\[
u_{k+1}=u_k-\eta\nabla E(u_k),
\qquad u_0=u.
\]
The standard descent lemma for $L$-smooth functions gives
\[
E(u_{k+1})
\le
E(u_k)
-\eta\Bigl(1-\tfrac{L\eta}{2}\Bigr)
\|\nabla E(u_k)\|_X^2.
\]
For $\eta\le 1/L$,
\[
E(u_{k+1})
\le
E(u_k)-\tfrac{\eta}{2}\|\nabla E(u_k)\|_X^2.
\]
Applying the PL inequality yields
\[
E(u_{k+1})
\le
(1-\mu\eta)E(u_k).
\]
Since $E$ is nonincreasing along the iteration, every iterate remains
in $\mathcal{N}$, so the descent and PL bounds apply at each step.
Iterating $K$ times gives
\[
E(R(u))\le (1-\mu\eta)^K E(u).
\]

Finally, using the lower and upper bounds relating $E$ to the distance
from $\mathcal{M}$,
\[
\begin{split}
\operatorname{dist}_X(R(u),\mathcal{M})^2
&\le \frac{1}{c_-}E(R(u)) \\
&\le
\frac{1}{c_-}(1-\mu\eta)^K E(u) \\
&\le
\frac{c_+}{c_-}(1-\mu\eta)^K
\operatorname{dist}_X(u,\mathcal{M})^2.
\end{split}
\]
Taking square roots completes the proof.
\end{proof}

\begin{theorem}[Conditional Manifold Stability]
Let
\[
r_j:=\operatorname{dist}_X(u_j,\mathcal{M})
\]
denote the constraint residual, and let $\varepsilon>0$. Let $\rho_K$ be
the contraction factor from Lemma~3. Assume that whenever
$r_j\le\varepsilon$, the predictor stays in the valid basin of attraction
of the score correction and satisfies
\[
r_{\widehat{j+1}}
\le L_h\varepsilon+Ch^q,
\qquad
r_{\widehat{j+1}}
:=
\operatorname{dist}_X(\widehat{u}_{j+1},\mathcal{M}).
\]
Assume further that
\[
\rho_K\bigl(L_h\varepsilon+Ch^q\bigr)\le\varepsilon
\]
and that $r_0\le\varepsilon$. Then the corrected scheme satisfies
\[
r_j\le\varepsilon
\qquad\text{for all } j\ge0.
\]
In particular, the bound is independent of the number of time steps.
\end{theorem}

\begin{proof}
We argue by induction on $j$. The base case is the hypothesis
$r_0\le\varepsilon$. Assume that $r_j\le\varepsilon$. By the predictor
estimate,
\[
r_{\widehat{j+1}}
\le L_h\varepsilon+Ch^q,
\]
and the provisional state lies in the valid basin of attraction.
By Lemma~3, the correction satisfies
\[
r_{j+1}
\le \rho_K r_{\widehat{j+1}}.
\]
Therefore,
\[
r_{j+1}
\le
\rho_K\bigl(L_h\varepsilon+Ch^q\bigr)
\le\varepsilon.
\]
Thus $r_j\le\varepsilon$ implies $r_{j+1}\le\varepsilon$, and the claim
follows by induction.
\end{proof}


\paragraph{Remark on Trajectory Accuracy vs. Manifold Distance.}
While Theorem 1 guarantees that the corrected numerical state remains close to the admissible manifold $\mathcal{M}$ (bounding the distance $r_j$ over arbitrary time horizons), this does not automatically guarantee pointwise trajectory accuracy over long time-spans. Pointwise trajectory accuracy is fundamentally dependent on the local truncation error of the combined predictor-corrector steps. An unstable integrator that is stabilized by a score model will remain on the physical manifold, but could still suffer from phase drift or diverge from the true trajectory of the exact initial state if the underlying time step is too large. Thus, the proposed framework focuses on structural and qualitative stability rather than guaranteeing exact pointwise tracking over long times.

\section{Additional Reaults}
\subsection{Additional Advection Results}
\label{app:additional_advection}

We provide extended results for the 1D advection experiments. We examine the evolution of two initial conditions: a smooth Gaussian pulse and a discontinuous two-box profile, to further demonstrate the structure-preserving capabilities of the score-corrected scheme.

\subsubsection{Gaussian Pulse}
The smooth Gaussian pulse initial condition is defined as:
\begin{equation}
    u_0(x) = \exp\!\left(-\frac{(x-x_0)^2}{2\beta^2}\right).
\end{equation}
This smooth profile serves as a standard sanity check where peak and shape preservation can be tracked over time. As shown in Figure~\ref{fig:transport_gaussian_snapshots}, the score-based correction successfully maintains the pulse's peak amplitude without introducing the numerical decay.

\begin{figure}[htbp]
    \centering
    \includegraphics[width=0.6\linewidth]{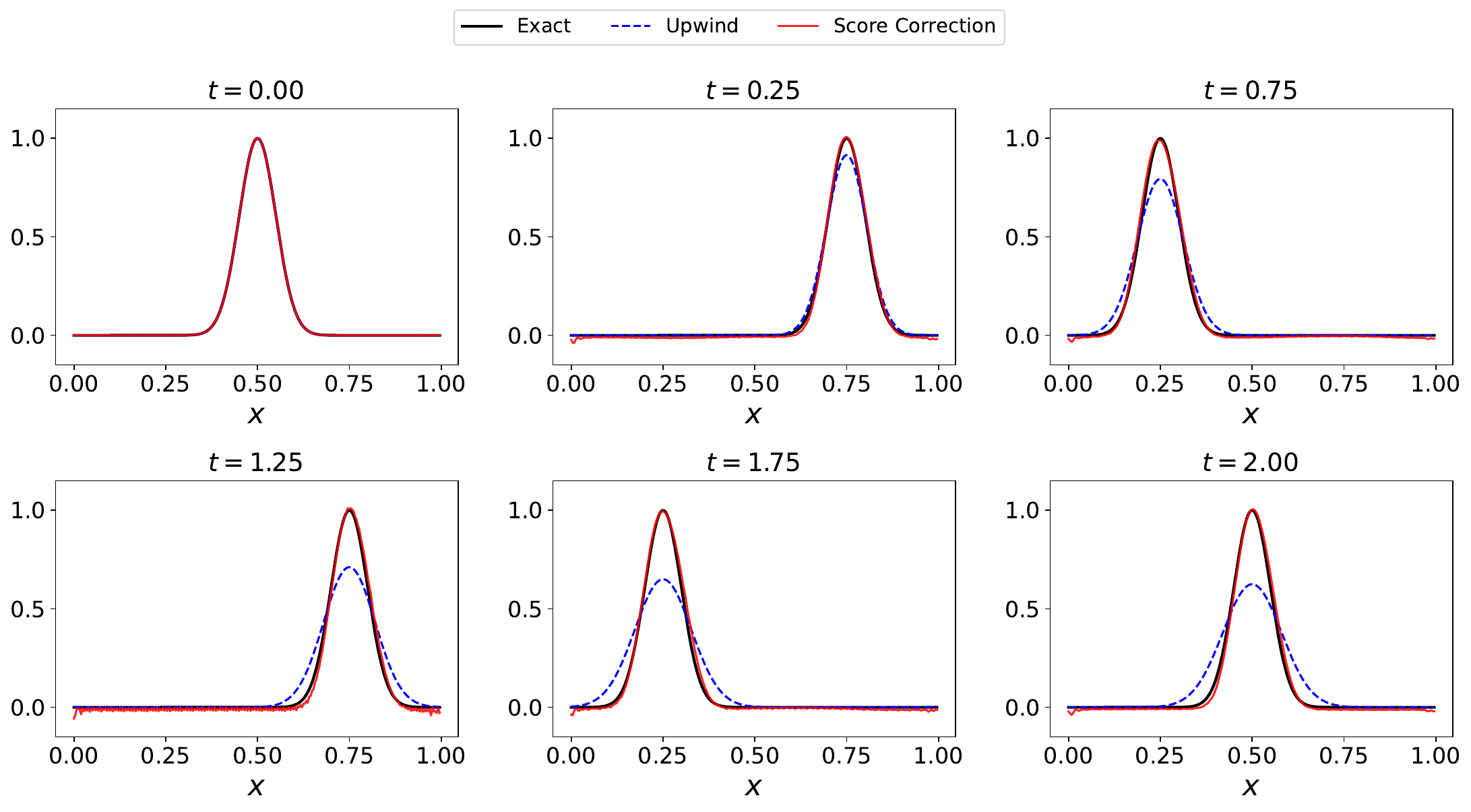}
    \caption{Gaussian pulse advection. The score-corrected method (red) does not exhibit decay, while the upwind method (blue, dashed) loses roughly $35\%$ of the peak by $t=2$.}
    \label{fig:transport_gaussian_snapshots}
\end{figure}

\subsubsection{Two-Box Pulse}
To test the method on multiple discontinuities, we utilize a two-box pulse defined by:
\begin{equation}
\begin{split}
    u_0(x) &= \mathbf{1}_{[x_1,\, x_1+w_+]}(x) \;-\; \mathbf{1}_{[x_2,\, x_2+w_-]}(x), \\
    &\quad x_2 > x_1 + w_+.
\end{split}
\end{equation}
A positive box ($+1$) and a negative box ($-1$) separated by a gap provide a stricter test of the numerical scheme. The score model must simultaneously maintain four sharp fronts, the correct plateau heights, and the proper gap between the pulses without generating spurious oscillations. Snapshots of this evolution are shown in Figure~\ref{fig:transport_twobox}.

\begin{figure}[htbp]
    \centering
    \includegraphics[width=0.6\linewidth]{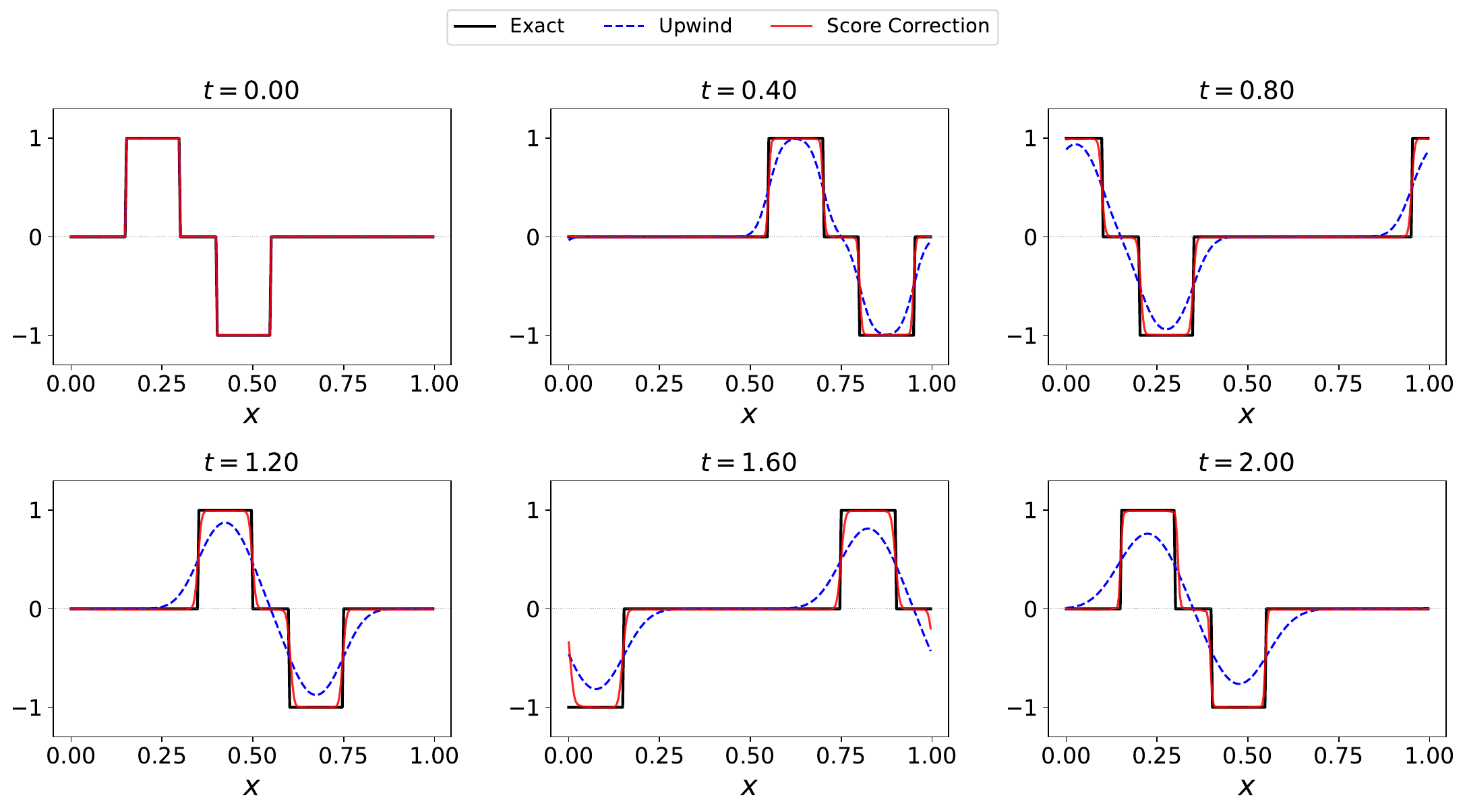}
    \caption{Two-box advection. The score-corrected method (red) preserves the amplitude and separation of both pulses.}
    \label{fig:transport_twobox}
\end{figure}

\subsubsection{Advection using PDEBench}\label{sec:advection_pdebench}

We use to the advection problem introduced in \Cref{subsec:phase1_advection}, using the version of this problem supplied by PDEBench \citep{PDEBench2022}. For the evaluation, we fix the domain length $L = 1$, the advection speed $a = 1$, and integrate to a final time $T = 2$, corresponding to two complete domain traversals. Initial conditions $u_0(x)$ are drawn from the PDEBench 1D-advection ensemble.
We reuse the first-order upwind discretization defined  in Section \ref{subsec:phase1_advection} as our baseline integrator $G_{\Delta t}$, operating here on an inference grid of $N = 256$ cells. 

Training samples are drawn directly from the PDEBench ensemble,
and a held-out ensemble of trajectories supplies the test samples.
To quantitatively evaluate performance, we compute the mean absolute error ($L^1$), the relative root-mean-square error (relative $L^2$), and the maximum pointwise error ($L^\infty$) across the spatial grid.

The score-based correction successfully stabilizes the numerical integration particularly in regimes where the baseline method fails. As detailed in \Cref{tab:adv_pdebench}, in the stable sub-critical regime, the corrected upwind scheme  reduces numerical diffusion compared to the baseline. In the supercritical regime where the baseline method diverges, the score correction successfully maintains long-term structural stability. Results are visualized in \Cref{fig:adv_paper_plot}, which demonstrates that the corrected scheme accurately tracks the analytical reference profile.

\begin{figure}[h]
    \centering
    \includegraphics[width=.7\linewidth]{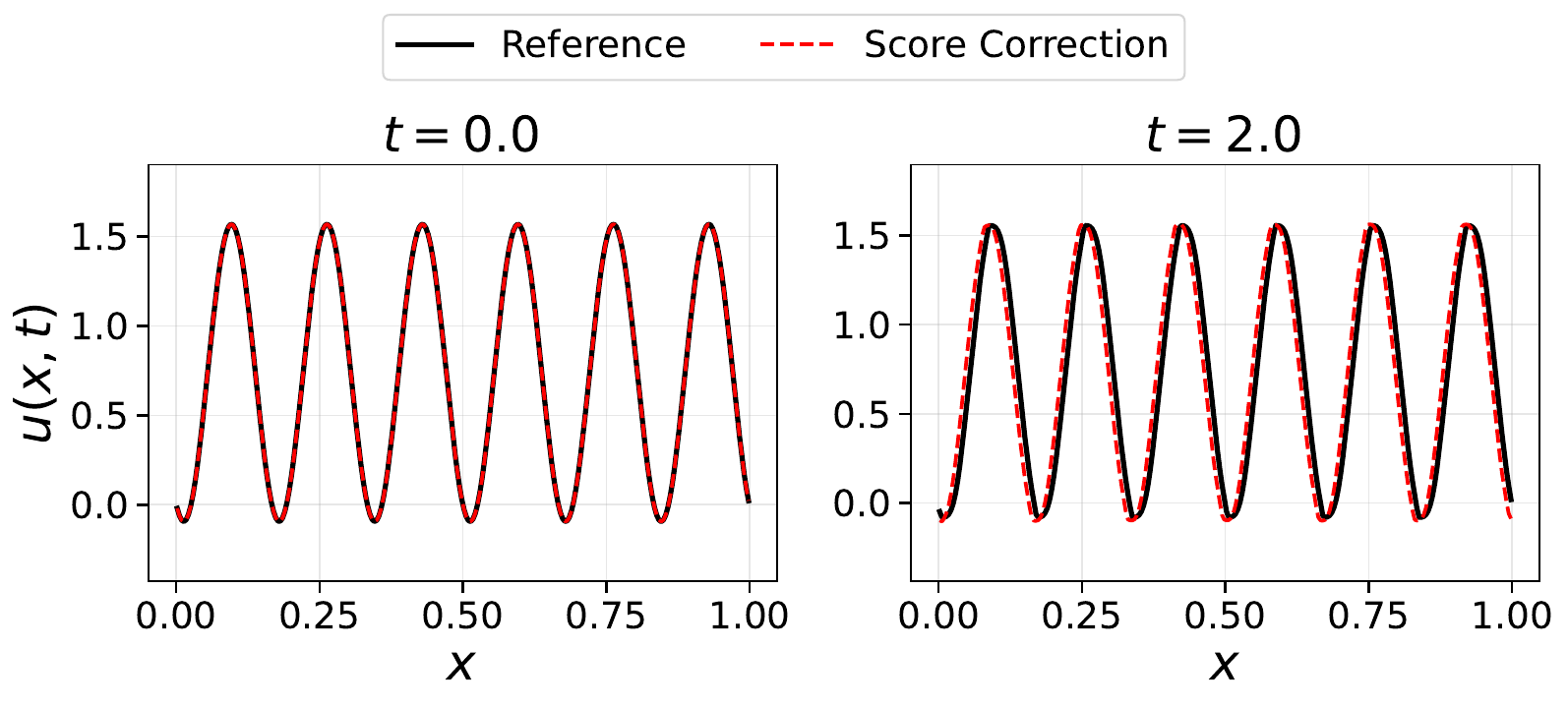}
    \caption{1D advection in a  supercritical regime ($\nu = 1.50$). The corrected scheme remains stable and tracks the
    analytical profile.}
    \label{fig:adv_paper_plot}
\end{figure}

\begin{table}[h]
    \centering
    \footnotesize
    \caption{1D advection errors across the PDEBench test ensemble. In the stable sub-critical regime, the score-corrected upwind scheme  reduces numerical errors. In the supercritical regime, the baseline upwind method diverges (\textit{div.}), whereas the score correction successfully maintains long-term structural stability.}
    \label{tab:adv_pdebench}
    \begin{tabular}{l | c c c}
        \hline
        method & $\mathrm{err}_{L^1}$
               & $\mathrm{err}_{L^2}$ 
               & $\mathrm{err}_{L^\infty}$ \\
        \hline
        \multicolumn{4}{c}{$\nu = 0.20$; $\Delta t = 7.81\!\times\!10^{-4}$} \\
        \hline
        Upwind                 & $2.45\!\times\!10^{-1}$ & $4.18\!\times\!10^{-1}$ & $4.98\!\times\!10^{-1}$ \\
        Score Correction       & $1.10\!\times\!10^{-1}$ & $2.08\!\times\!10^{-1}$ & $2.84\!\times\!10^{-1}$ \\
        \hline
        \multicolumn{4}{c}{$\nu = 0.50$; $\Delta t = 1.95\!\times\!10^{-3}$} \\
        \hline
        Upwind                 & $1.98\!\times\!10^{-1}$ & $3.42\!\times\!10^{-1}$ & $4.06\!\times\!10^{-1}$ \\
        Score Correction       & $7.11\!\times\!10^{-2}$ & $1.36\!\times\!10^{-1}$ & $1.95\!\times\!10^{-1}$ \\
        \hline
        \multicolumn{4}{c}{$\nu = 0.80$; $\Delta t = 3.13\!\times\!10^{-3}$} \\
        \hline
        Upwind                 & $1.13\!\times\!10^{-1}$ & $2.01\!\times\!10^{-1}$ & $2.44\!\times\!10^{-1}$ \\
        Score Correction       & $3.40\!\times\!10^{-2}$ & $7.21\!\times\!10^{-2}$ & $1.10\!\times\!10^{-1}$ \\
        \hline
        \multicolumn{4}{c}{$\nu = 1.10$; $\Delta t = 4.30\!\times\!10^{-3}$} \\
        \hline
        Upwind                 & \textit{div.} & \textit{div.}  & \textit{div.} \\
        Score Correction       & $5.15\!\times\!10^{-2}$ & $1.19\!\times\!10^{-1}$ & $1.75\!\times\!10^{-1}$ \\
        \hline
        \multicolumn{4}{c}{$\nu = 1.50$; $\Delta t = 5.86\!\times\!10^{-3}$} \\
        \hline
        Upwind                 & \textit{div.} & \textit{div.} & \textit{div.} \\
        Score Correction       & $1.31\!\times\!10^{-1}$ & $2.73\!\times\!10^{-1}$ & $3.77\!\times\!10^{-1}$ \\
        \hline
    \end{tabular}
\end{table}

\subsection{NLS Visualization}
We provide visual results demonstrating the efficacy of our score-based correction method applied to the nonlinear Schrödinger (NLS) equation. As illustrated in \Cref{fig:nls_marginal_paper}, the correction successfully eliminates the nonphysical grid-scale ripples that accumulate in the baseline numerical method, maintaining a clean soliton profile over the integration window.

\begin{figure}[h]
    \centering
    \includegraphics[width=.8\linewidth]{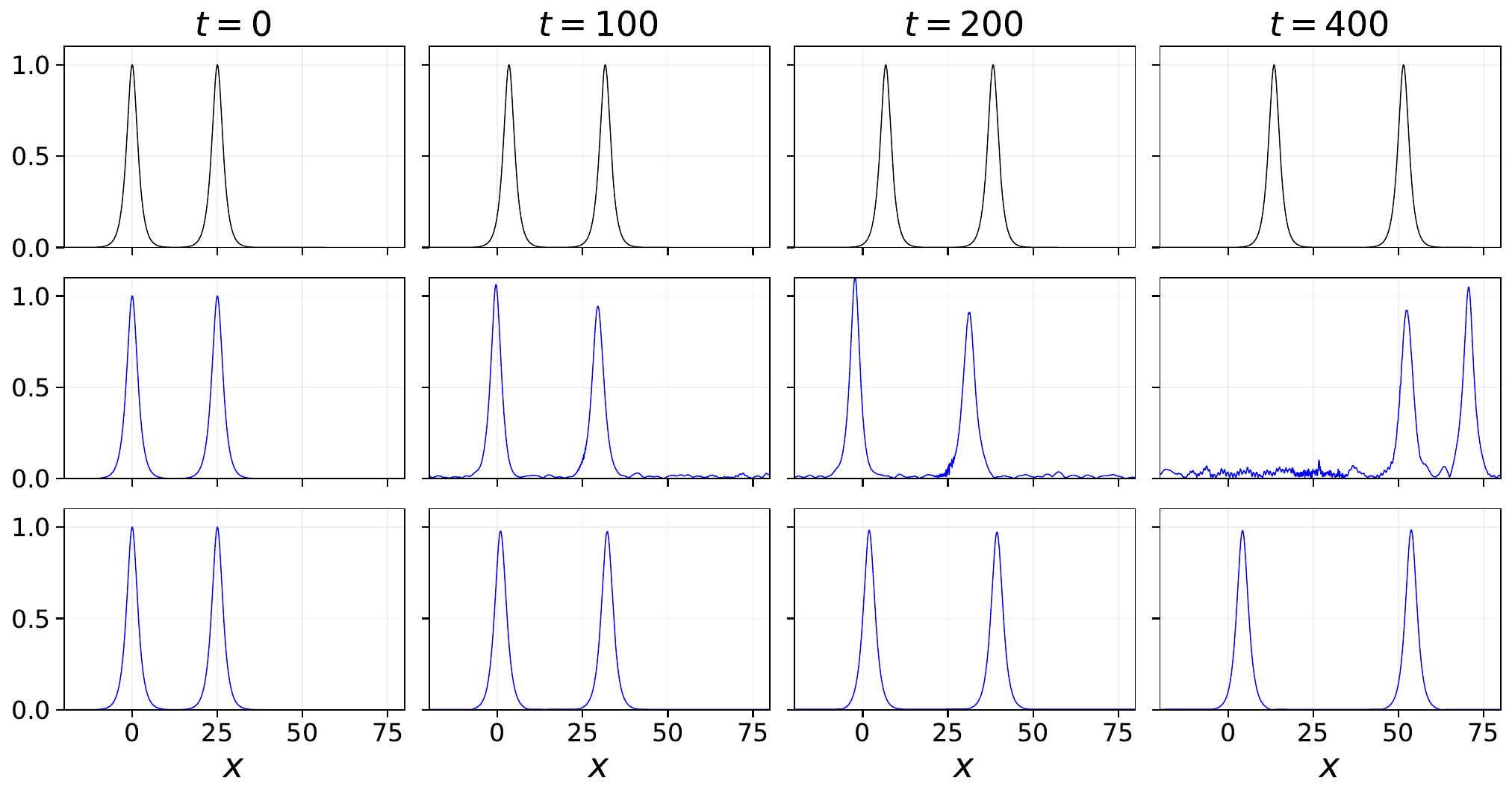}
\caption{NLS results for $\Delta t = 0.2,\,\Delta x = 0.1$. Each panel plots
    $|\psi(x)|$ at the labeled time. \emph{Top:} fine-grid reference;
    \emph{middle:} baseline numerical method; \emph{bottom:} score correction.
    The baseline develops grid-scale ripples that the score correction removes
    while preserving the soliton profile.}

    \label{fig:nls_marginal_paper}
\end{figure}

\subsection{Burgers Visualization}
We provide space-time visualizations of the Burgers' equation to further illustrate the stabilizing effects of our method. As shown in \Cref{fig:burgers_spacetime}, while the uncorrected forward-time centered-space (FTCS) baseline struggles to maintain stability and diverges, the score-corrected scheme tracks the underlying dynamics of the reference trajectory over the full integration window.

\begin{figure}[h]
    \centering
    \includegraphics[width=.8\linewidth]{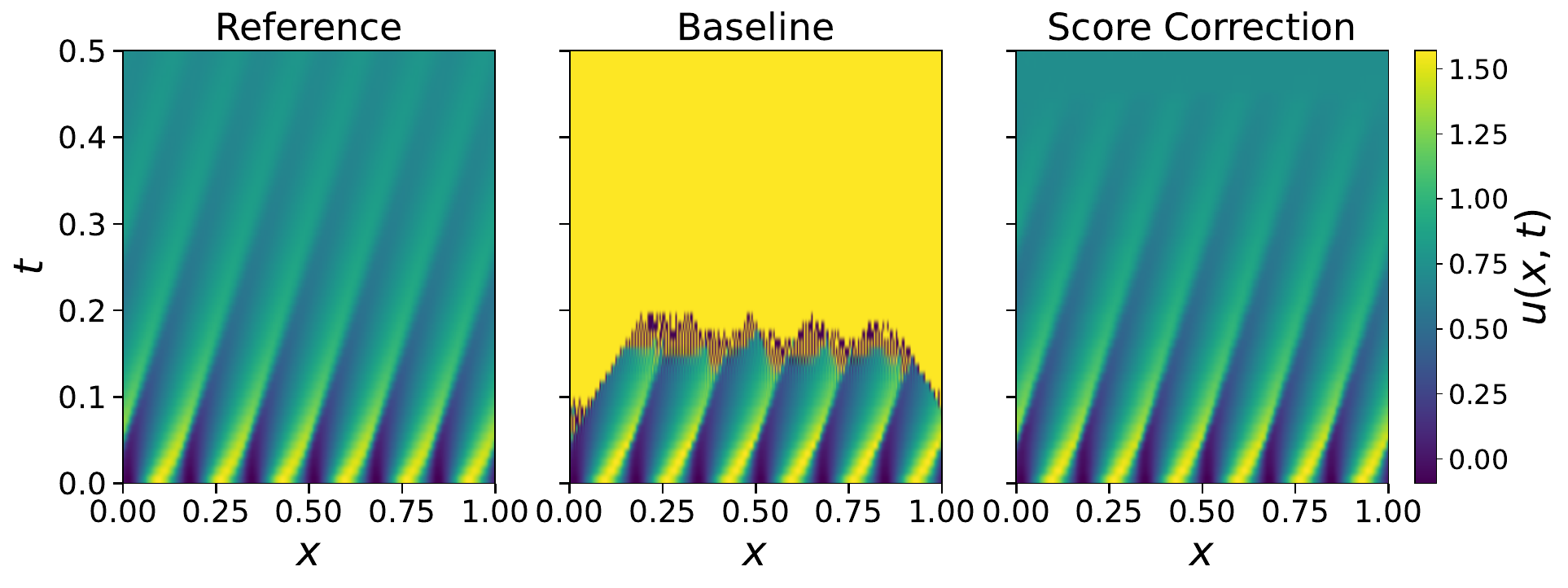}
    \caption{1D  Burgers results, for a
    representative PDEBench trajectory.
    \emph{Left:} PDEBench reference. \emph{Center:} baseline
    FTCS. \emph{Right:} the score-corrected
    scheme, which tracks the reference.}
    \label{fig:burgers_spacetime}
\end{figure}

\section{Classical Stabilization Baselines}
\label{app:classical}

This appendix describes the classical stabilization baselines used for comparison
with the proposed score-based correction method. Each baseline is implemented at
the same spatial resolution and timestep as the corresponding experiment.

\subsection{WENO-5 (Weighted Essentially Non-Oscillatory)}
\label{app:weno}

We implement the fifth-order finite-difference WENO scheme of
\citet{jiang1996efficient} with global Lax--Friedrichs flux splitting. For a
conservation law written in the form
\[
u_t + f(u)_x = (\text{additional diffusive/dispersive terms}),
\]
the convective flux is split as
\[
f^{\pm}(u)=\frac12\left(f(u)\pm a u\right),
\qquad
a=\max_x |f'(u)|.
\]
The positive flux component is reconstructed using a left-biased stencil,
while the negative component is obtained by symmetry. The interface flux is
given by
\[
F_{j+1/2}=\sum_{k=0}^{2}\omega_k\hat f^{(k)},
\]
where the candidate reconstructions are
\begin{align*}
\hat f^{(0)}
&=\frac16(2f_{j-2}-7f_{j-1}+11f_j),\\
\hat f^{(1)}
&=\frac16(-f_{j-1}+5f_j+2f_{j+1}),\\
\hat f^{(2)}
&=\frac16(2f_j+5f_{j+1}-f_{j+2}).
\end{align*}
The nonlinear weights are computed as
\[
\omega_k=\frac{\alpha_k}{\sum_l\alpha_l},
\qquad
\alpha_k=\frac{d_k}{(\varepsilon+\beta_k)^2},
\]
with linear weights
\[
(d_0,d_1,d_2)=
\left(\frac1{10},\frac6{10},\frac3{10}\right),
\qquad
\varepsilon=10^{-6}.
\]
The Jiang--Shu smoothness indicators are
\begin{align*}
\beta_0 &=
\frac{13}{12}(f_{j-2}-2f_{j-1}+f_j)^2
+\frac14(f_{j-2}-4f_{j-1}+3f_j)^2,\\
\beta_1 &=
\frac{13}{12}(f_{j-1}-2f_j+f_{j+1})^2
+\frac14(f_{j-1}-f_{j+1})^2,\\
\beta_2 &=
\frac{13}{12}(f_j-2f_{j+1}+f_{j+2})^2
+\frac14(3f_j-4f_{j+1}+f_{j+2})^2 .
\end{align*}
The resulting semi-discrete operator, including the remaining physical terms,
is advanced using the third-order strong-stability-preserving Runge--Kutta
scheme (SSP-RK3).

\paragraph{KdV.}
The convective flux is
\[
f(u)=-\beta u^2=\frac12u^2,
\]
with $\beta=-1/2$ and
$a=\max_x|u|$. The dispersive contribution is discretized using the centered
finite-difference approximation
\[
u_{xxx}(x_j)\approx
\frac{
u_{j+2}-2u_{j+1}+2u_{j-1}-u_{j-2}
}
{2\Delta x^3},
\]
and treated explicitly. The discretization uses $N=200$ grid points and
$\Delta t=5\times10^{-4}$.

\paragraph{Burgers.}
We use
\[
f(u)=\frac12u^2,
\]
with the viscous contribution discretized using the standard centered second
difference,
\[
u_{xx}(x_j)\approx
\frac{u_{j+1}-2u_j+u_{j-1}}{\Delta x^2}.
\]
The viscous term is treated explicitly and the experiments are performed on
$N=256$ grid points.

\subsection{TVD (Total Variation Diminishing), Superbee Limiter}
\label{app:tvd}

\paragraph{Burgers.}
We use a MUSCL reconstruction with a Superbee slope limiter and a local
Lax--Friedrichs numerical flux, advanced in time by the second-order
strong-stability-preserving Runge--Kutta scheme (SSP-RK2). Defining the
backward and forward differences
\[
a=u_j-u_{j-1},
\qquad
b=u_{j+1}-u_j,
\]
the limited slope is
\[
d_j=
\operatorname{maxmod}
\left(
\operatorname{minmod}(a,2b),
\operatorname{minmod}(2a,b)
\right),
\]
where $\operatorname{minmod}$ and $\operatorname{maxmod}$ select the argument
with smaller and larger magnitude, respectively, when all arguments have the
same sign, and return zero otherwise. The reconstructed interface states are
\[
u^L_j=u_j+\frac12d_j,
\qquad
u^R_j=u_{j+1}-\frac12d_{j+1},
\]
and the numerical flux is
\[
F_{j+1/2}
=
\frac12
\left(
f(u^L_j)+f(u^R_j)
\right)
-\frac12\,\alpha\,(u^R_j-u^L_j),
\]
where $\alpha=\max_x|u|$ is the (global) maximum wave speed. The viscous
contribution is discretized using the centered second difference and
incorporated explicitly.



\subsection{Spectral Filtering}
\label{app:spectral}

\paragraph{KdV.}
We use a pseudo-spectral discretization in space.
Spatial derivatives are evaluated in Fourier space, while the linear dispersive
operator is integrated using an integrating factor. Denoting the Fourier
transform of $u$ by $\hat u$, the evolution equation is written as
\[
\hat u_t
=
\beta(ik)\widehat{u^2}
+
\nu(ik)^3\hat u .
\]
The linear operator is therefore
\[
L(k)=-i\nu k^3,
\]
and we define the integrating factor
\[
E(k)=\exp(L(k)\Delta t).
\]
The nonlinear contribution is advanced using a fourth-order integrating-factor
Runge--Kutta method (IF-RK4). After each time step, a spectral low-pass filter
is applied by setting the highest-frequency modes to zero:
\[
\hat u(k)=0,
\qquad
|k|>0.9\,k_{\max}.
\]
The discretization uses $N=200$ grid points and
$\Delta t=5\times10^{-4}$.

\paragraph{NLS.}
The spectral filter is applied on top
of the split-step baseline. We use a smooth
exponential filter rather than a sharp truncation:
\[
\hat u(k)\leftarrow \hat u(k)\sigma(k),
\]
where
\[
\sigma(k)=
\exp\left(
-\alpha
\left(
\frac{|k|}{k_{\max}}
\right)^p
\right),
\]
with parameters
\[
\alpha=36,
\qquad
p=16.
\]
This filter preserves the low-frequency components while damping the highest
spectral modes.

\paragraph{Burgers.}
The spectral filter is applied as a stabilization step
to the FTCS baseline. After each FTCS update, the solution is transformed to
Fourier space, a two-thirds dealiasing filter is applied, and the solution is
transformed back to physical space. The FTCS instability in this regime
occupies a broad high-frequency band that extends below the filter's cutoff;
consequently, the filter delays divergence at the larger timesteps but only
stabilizes the integration at the smallest one.



\section{Ablation Study}
\label{app:ablation}

To quantify the contribution of each component of the proposed corrector and
to compare it with alternative learned correction strategies, we perform an
ablation study on the KdV equation. All configurations use the same
conservative leap-frog (Zabusky--Kruskal) discretization on an $N=200$ grid
with $\Delta t=5\times10^{-4}$, and are integrated until $t=10$. Accuracy is
measured using the relative pointwise error and the relative drift in the $L^2$ norm.

We first consider the uncorrected leap-frog scheme, which serves as the
baseline without any additional stabilization. We then evaluate an invariant
projection method, where each leap-frog update is followed by an exact
projection onto the mass, $L^2$-norm, and Hamiltonian invariant manifolds.
This projection approach does not include any learned correction. We also
evaluate the score correction alone, where the score update is performed
without invariant projection.

As alternative learned correction strategies, we consider a denoising
autoencoder (DAE) and a supervised residual correction model. The DAE
$f_\theta$ is trained to reconstruct clean states from Gaussian-corrupted
inputs by minimizing
\[
\min_\theta
\mathbb{E}\left[
\left\|f_\theta(u+\varepsilon)-u\right\|^2
\right],
\qquad
\varepsilon\sim\mathcal{N}(0,\sigma^2I).
\]
Unlike the score model, which estimates the gradient of the log-density of the
perturbed data distribution, the DAE directly predicts a point estimate of the
clean state. During inference, the provisional leap-frog update
$\hat u_{n+1}$ is replaced by
\[
u_{n+1}=f_\theta(\hat u_{n+1}).
\]

The residual correction model $g_\theta$ is trained using supervised
regression to predict the one-step discrepancy between a coarse leap-frog
update and the corresponding fine-grid solution. Given the current state and
the provisional update, the correction is applied as
\[
u_{n+1}
=
\hat u_{n+1}+g_\theta(\hat u_{n+1},u_n).
\]
Unlike the score-based correction, this approach performs deterministic
residual regression and does not explicitly incorporate a learned
representation of the admissible solution distribution.

Finally, we evaluate the complete proposed method, which combines the score
correction with invariant enforcement.

The results are summarized in \Cref{tab:ablation_study}. The uncorrected
leap-frog scheme becomes unstable before the final integration time, and the
learned residual correction exhibits a similar failure due to the accumulation
of autoregressive errors outside the training distribution.

The remaining methods remain stable, but exhibit different trade-offs between
accuracy and invariant preservation. Invariant projection enforces the
conservation laws but does not sufficiently reduce trajectory error,
whereas the score correction substantially improves accuracy while introducing
a small invariant drift. The DAE correction performs less effectively,
indicating that direct state reconstruction is less suitable for long-time
stabilization than score-based correction.

Combining score correction with invariant projection achieves the best overall
performance, preserving the invariants while attaining a low
pointwise error. This demonstrates the complementary roles of the two
components: the score correction provides an accurate learned stabilization
mechanism, and the projection step enforces the physical constraints of the
dynamics.

\begin{table}[h]
\centering
\caption{Ablation study on the KdV equation. Relative norm drift
($\mathrm{err}_{I_2}$) and relative pointwise error
($\mathrm{err}_{\mathrm{pt}}$) at $t=10$ with $\Delta t=5\times10^{-4}$ for
different corrector configurations. \emph{div.} denotes divergence before the
final integration time.}
\label{tab:ablation_study}
\begin{tabular}{l c c}
\toprule
\textbf{Configuration} & $\mathrm{err}_{I_2}$ & $\mathrm{err}_{\mathrm{pt}}$ \\
\midrule
Leap-frog baseline & \textit{div.} & \textit{div.} \\
Invariant projection & $0$ & $1.41$ \\
Score step only & $1.6\times10^{-3}$ & $0.71$ \\
Denoising autoencoder & $1.2\times10^{-1}$ & $1.35$ \\
Learned residual correction & \textit{div.} & \textit{div.} \\
Full score correction & $0$ & $0.70$ \\
\bottomrule
\end{tabular}
\end{table}

\section{KdV Timestep Refinement}
\label{app:kdv_refine}

We examine whether the instability of the KdV leap-frog baseline
scheme can be mitigated through timestep refinement. \Cref{tab:kdv_refine}
reports the blow-up time as $\Delta t$ is progressively reduced. Although
smaller timesteps delay the onset of instability, the scheme remains unstable
and does not reach the target integration time $t=10$. Even after a $50\times$
reduction in timestep, the blow-up time increases only from approximately $t=2.34$ to
$t=3.37$. These results indicate that timestep refinement alone does not
provide a practical stabilization mechanism for this discretization, and thus enphisize the advantege of using the score correction scheme.

\begin{table}[h]
\centering
\caption{Stability of the KdV leap-frog baseline for various timesteps sizes. Smaller timesteps delay instability but do not prevent divergence.}
\label{tab:kdv_refine}
\begin{tabular}{lcccccc}
\toprule
$\Delta t$
& $5{\times}10^{-4}$
& $2{\times}10^{-4}$
& $1{\times}10^{-4}$
& $5{\times}10^{-5}$
& $2{\times}10^{-5}$
& $1{\times}10^{-5}$ \\
\midrule
Steps to $t=10$
& 20k
& 50k
& 100k
& 200k
& 500k
& 1M \\
Blow-up time
& 2.34
& 2.38
& 2.75
& 2.97
& 3.10
& 3.37 \\
\bottomrule
\end{tabular}
\end{table}

\section{Execution Time}
\label{app:timing}

To assess the computational overhead of the proposed method, we compare the per-step execution time of the score-based stabilization against the classical WENO-5 \citep{jiang1996efficient} baseline. All timing evaluations were performed on standard hardware (NVIDIA GeForce RTX 4090 GPUs) using a batch size of 128 trajectories to leverage parallel processing.

The WENO-5 scheme is computationally demanding per integration step. It requires multiple candidate stencil reconstructions and the evaluation of nonlinear smoothness indicators. Furthermore, because it is typically advanced in time using a multi-stage Runge-Kutta method (such as SSP-RK3), the heavy spatial operator must be evaluated multiple times for a single step forward in time. 

In contrast, the score-corrected scheme is highly parallelizable and consists of an inexpensive provisional update followed by a single forward pass through the learned U-Net stabilization operator. As shown in \Cref{tab:timing}, the score-based method executes faster per step than the WENO-5 scheme.

\begin{table}[htbp]
\centering
\caption{Per-step execution time, reported as the average time per trajectory per integration step for the two stabilization approaches. The score correction includes the cost of the baseline time-stepping update and the learned correction network evaluation.}
\label{tab:timing}
\begin{tabular}{l c c}
\toprule
Equation & WENO-5 (ms) & Score Correction (ms) \\
\midrule
KdV & 15.05 & 9.01 \\
Burgers & 14.78 & 12.64 \\
\bottomrule
\end{tabular}
\end{table}

\section{Model Generalization}
\label{app:generalization}

To evaluate the transferability of the learned stabilization mechanism, we test
the trained score model on initial conditions that are not included during
training. In all experiments, the trained model is kept fixed and inference is
performed without retraining or fine-tuning. The baseline integrator
$G_{\Delta t}$, timestep, spatial discretization, and invariant projection
parameters (when applicable) remain unchanged. For each unseen initial condition, an independent high-accuracy
spectral reference solution is generated for evaluation.

\subsection{Korteweg--de Vries}
\label{app:gen_kdv}

The score model is trained using the initial condition $u(x,0)=\cos(\pi x)$.
We evaluate
the resulting method on two initial conditions outside the training
distribution.



\paragraph{Higher-frequency cosine.}
We consider the unseen initial condition
\begin{equation}
    u(x,0)=\cos(2\pi x).
\end{equation}
This is a higher-frequency cosine mode compared with the training initial
condition. The uncorrected leap-frog scheme becomes unstable at approximately
$t=2.19$, whereas the score-corrected method remains stable over the full
integration interval. As shown in
\Cref{fig:kdv_gen_cos2pi}, the corrected solution reproduces the qualitative
dispersive dynamics of the reference solution throughout the evolution.

\paragraph{Two-soliton state.}
We next consider a superposition of two solitons with different amplitudes,
\begin{equation}
    u(x,0)=
    2\,\mathrm{sech}^2\!\Big(\frac{x-x_1}{w(2)}\Big)
    +
    \mathrm{sech}^2\!\Big(\frac{x-x_2}{w(1)}\Big),
    \qquad
    w(A)=2\sqrt{\frac{3\delta^2}{A}},
\end{equation}
with centers $x_1=0.5$ and $x_2=1.3$. The score-corrected method remains
stable and captures the soliton interaction dynamics. The corresponding
evolution is shown in \Cref{fig:kdv_gen_two_soliton}, where
the corrected solution follows the reference soliton collision dynamics.


\begin{figure}[t]
    \centering
    \begin{subfigure}{0.49\linewidth}
        \centering
        \includegraphics[width=\linewidth]{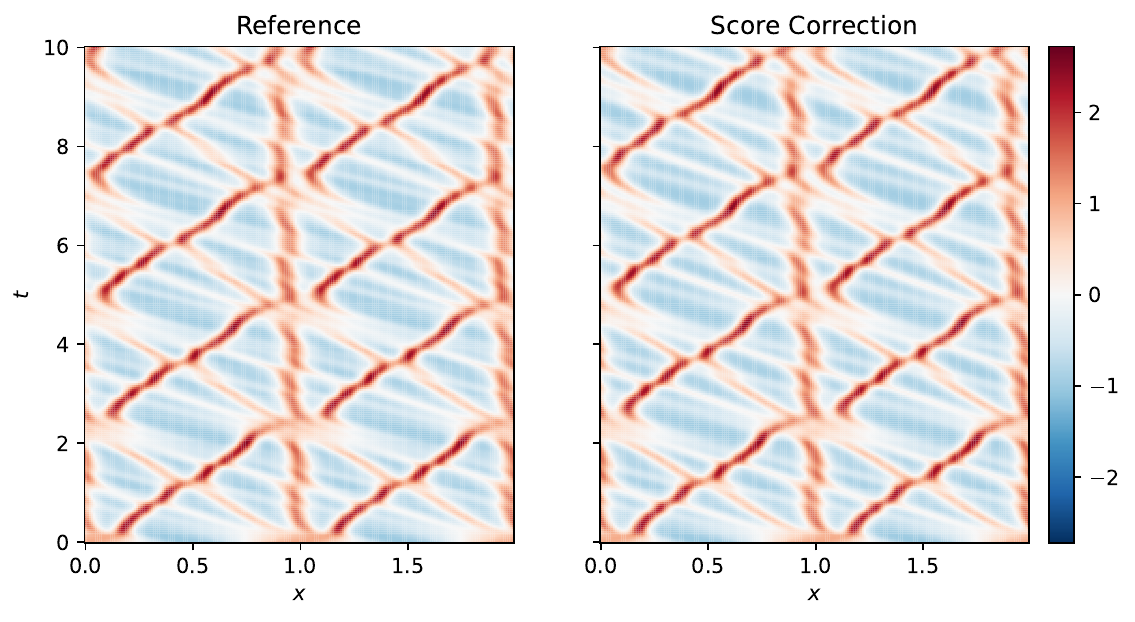}
        \caption{Higher-frequency cosine.}
        \label{fig:kdv_gen_cos2pi}
    \end{subfigure}
    \hfill
    \begin{subfigure}{0.49\linewidth}
        \centering
        \includegraphics[width=\linewidth]{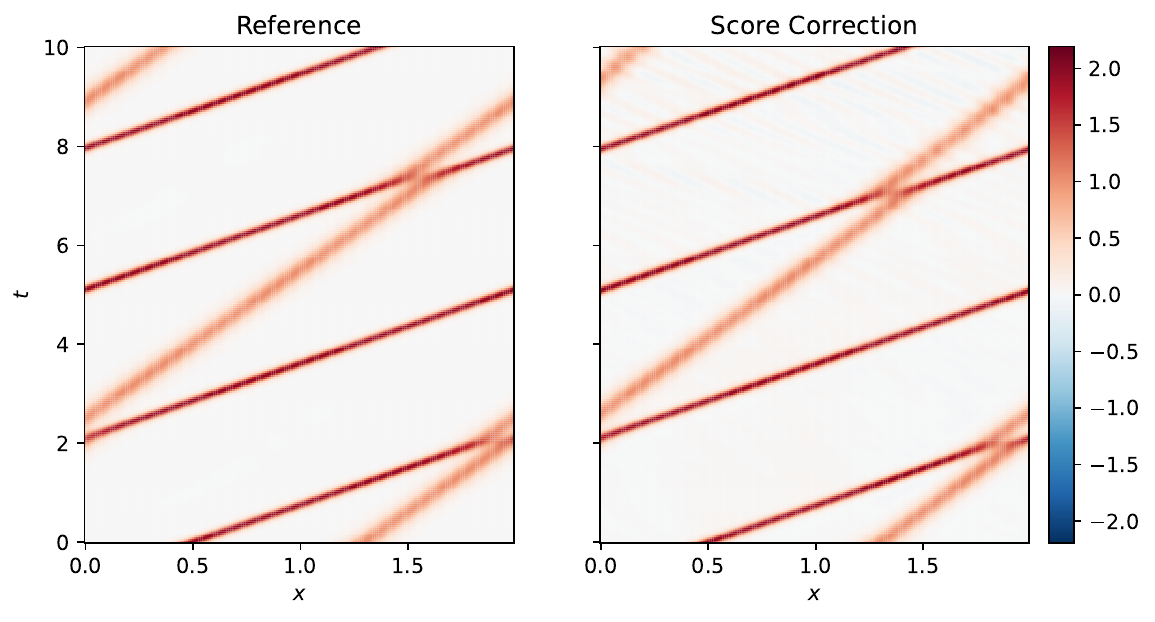}
        \caption{Two-soliton.}
        \label{fig:kdv_gen_two_soliton}
    \end{subfigure}
    \caption{Generalization of the KdV score correction to unseen initial
    conditions. \textbf{Left:} $u(x,0)=\cos(2\pi x)$.
    \textbf{Right:} two-soliton initial condition. In both cases, the
    corrected evolution preserves the relevant dynamics.}
    \label{fig:kdv_generalization}
\end{figure}

\subsection{Burgers' Equation}
\label{app:gen_burgers}

The Burgers score model is trained on the PDEBench initial-condition
distribution, where each sample is generated as a superposition of sinusoidal
modes,
\begin{equation}
    u(x,0) = s\, W(x) \sum_{n=1}^{8} a_n \sin\!\big(k_n x + \phi_n\big),
    \qquad k_n = 2\pi n .
\end{equation}
The active wavenumbers are sampled uniformly from $\{1,\dots,8\}$, with
amplitudes $a_n\sim\mathcal{U}(0,1)$ and phases
$\phi_n\sim\mathcal{U}(0,2\pi)$. The distribution additionally includes a
random constant offset, random sign $s\in\{-1,+1\}$, and, with low probability,
a pointwise absolute value and smooth localization window $W(x)$. To evaluate
generalization beyond the training distribution, we consider two qualitatively
different initial conditions.

\paragraph{Localized Gaussian pulse.}
We first evaluate the model on a smooth localized profile,
\begin{equation}
    u(x,0) = \exp\!\Big(-\frac{(x-\tfrac12)^2}{2\sigma^2}\Big),
    \qquad \sigma = 0.08 .
\end{equation}
Where the uncorrected FTCS discretization becomes
unstable at approximately $t=0.11$, while the score-corrected scheme remains
stable throughout the integration interval. The corrected trajectory captures
the nonlinear steepening and propagation of the pulse, closely matching the
reference solution as shown in \Cref{fig:burgers_gen_gauss}.

\paragraph{High-frequency mode.}
We further test an oscillatory initial condition,
\begin{equation}
    u(x,0)=\sin(4\pi x).
\end{equation}
Despite the higher-frequency content, the score correction prevents the instability observed
in the FTCS solution. The corrected evolution remains stable and
reproduces the reference shock formation, as presented in \Cref{fig:burgers_gen_sin4pi}.

\begin{figure}[h]
    \centering
    \begin{subfigure}{0.49\linewidth}
        \centering
        \includegraphics[width=\linewidth]{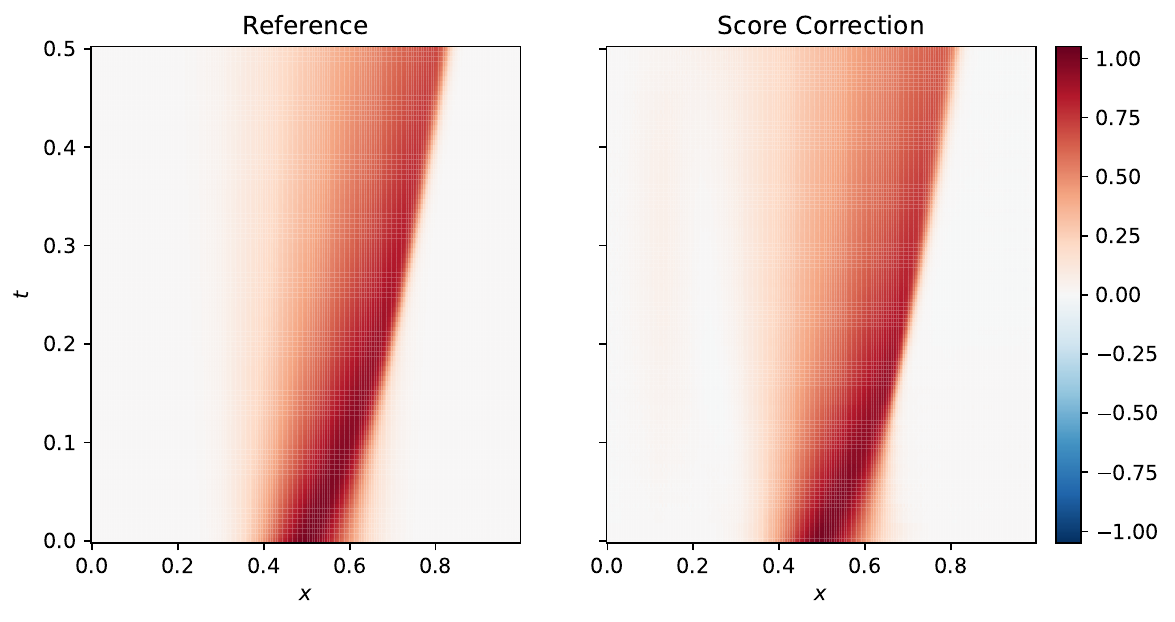}
        \caption{Localized Gaussian pulse.}
        \label{fig:burgers_gen_gauss}
    \end{subfigure}
    \hfill
    \begin{subfigure}{0.49\linewidth}
        \centering
        \includegraphics[width=\linewidth]{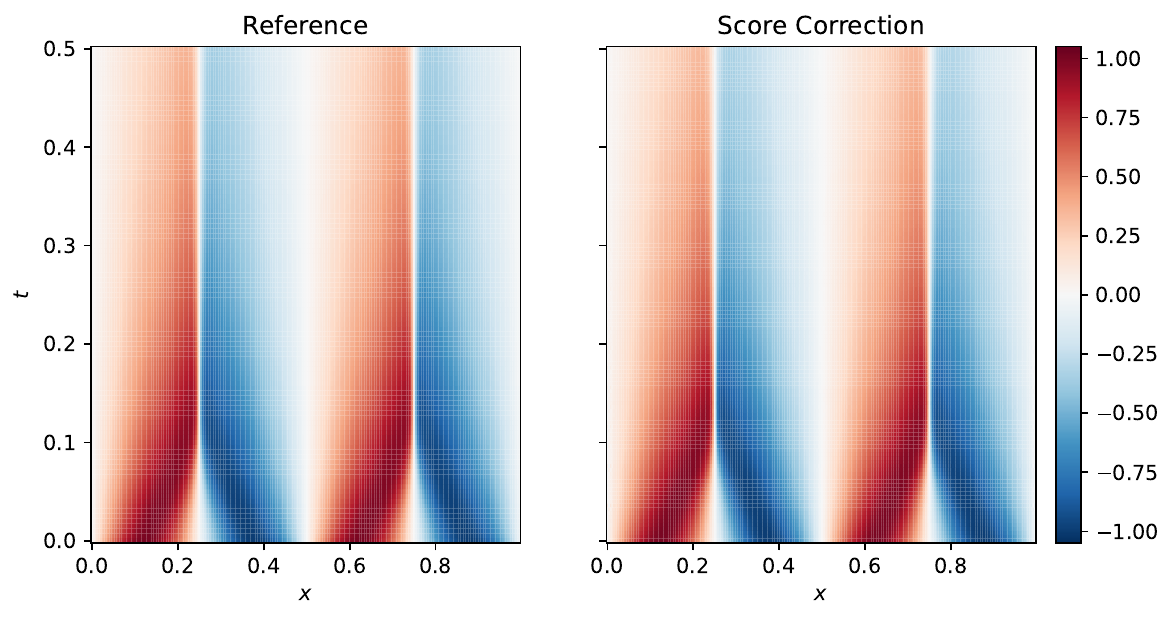}
        \caption{Higher-frequency mode.}
        \label{fig:burgers_gen_sin4pi}
    \end{subfigure}
\caption{Generalization of the Burgers score correction to unseen initial
conditions. \textbf{Left:} localized Gaussian pulse.
\textbf{Right:} higher-frequency mode $u(x,0)=\sin(4\pi x)$. In both cases, the
corrected evolution remains stable and preserves the reference dynamics.}
    \label{fig:burgers_generalization}
\end{figure}

\section{Experimental Details}
\label{app:experimental-details}

We provide the experimental details of the results presented in this paper in \Cref{tab:exp-details-advection} and \Cref{tab:exp-details-others}. All experiments were conducted using NVIDIA GeForce RTX 4090 GPUs. 
For hyperparameter selection,
the noise level $\sigma$, the step size $\eta$, and
the correction interval $K$ were chosen by grid search over
$\sigma\in\{0.02,0.05,0.1,0.2,0.3,0.5,0.8\}$,
$\eta\in\{10^{-5},10^{-4},3{\times}10^{-4},10^{-3},3{\times}10^{-3},10^{-2},3{\times}10^{-2}\}$,
and $K\in\{1,5,10,20\}$. The configuration minimizing the relative $L^2$ error on the
held-out trajectories.
All stochastic components use fixed seeds. Network initialization, the
score noise, and minibatch shuffling are seeded at the start
of training. The PDEBench samples
are generated with a fixed key.
For the PDEBench advection and Burgers experiments, all reported errors are averaged over a held-out ensemble of
$100$ unseen trajectories. The KdV and NLS errors are computed against high-accuracy spectral
reference solutions. We report ensemble-mean errors.

\begin{table*}[h]
\centering
\setlength{\tabcolsep}{4pt}
\renewcommand{\arraystretch}{1.15}
\caption{Experimental details for Advection.}
\label{tab:exp-details-advection}
\begin{tabular}{l cccc}
\toprule
& Gaussian & Box & Two-Box & PDEBench \\
\midrule
\multicolumn{5}{l}{\textit{Problem setup}}\\
Domain (periodic)    & $[0,1]$ & $[0,1]$ & $[0,1]$ & $[0,1]$ \\
Grid $N$ ($\Delta x$) & $256\,(1/256)$ & $256\,(1/256)$ & $256\,(1/256)$ & $256\,(1/256)$ \\
Base scheme          & Upwind & Upwind & Upwind & Upwind \\
Timestep $\Delta t$  & $1.95{\times}10^{-3}$ & $1.95{\times}10^{-3}$ & $1.95{\times}10^{-3}$ & $5{\times}10^{-3}$ \\
Reference resolution     & $2048$ & $2048$ & $2048$ & $1024$ \\
\midrule
\multicolumn{5}{l}{\textit{Score network}}\\
Architecture         & \multicolumn{4}{c}{U-Net \citep{ronneberger2015u}} \\
Levels               & $3$ & $3$ & $3$ & $5$ \\
Channels             & $(16,32,64)$ & $(32,64,128)$ & $(32,64,128)$ & $(32,64,128,256,512)$ \\
Parameters           & $0.18$M & $0.62$M & $0.62$M & $2.91$M \\
\midrule
\multicolumn{5}{l}{\textit{Training}}\\
Optimizer            & \multicolumn{4}{c}{Adam \citep{kingma2014adam}} \\
Learning rate        & $2{\times}10^{-4}$ & $2{\times}10^{-4}$ & $2{\times}10^{-4}$ & $2{\times}10^{-4}$ \\
LR schedule          & Cosine & Cosine & Cosine & Cosine \\
Batch size           & $512$ & $512$ & $512$ & $512$ \\
Epochs               & $500$ & $500$ & $500$ & $500$ \\
Training-data resolution & $2048$ & $2048$ & $2048$ & $1024$ \\
Training-data $\Delta t$ & Exact & Exact & Exact & Exact \\
\midrule
\multicolumn{5}{l}{\textit{Correction (inference)}}\\
$\sigma$             & $0.1$ & $0.05$ & $0.2$ & $0.05$ \\
Step $\eta$        & $0.003$ & $0.03$ & $0.05$ & $0.003$ \\
Interval $K$         & $5$ & $10$ & $20$ & $10$ \\
Enforced invariants  & None & Clamp $[0,1]$ & Clamp $+$ Mass & Mass $+\,L^2$ \\
\midrule
\multicolumn{5}{l}{\textit{Initial conditions (training)}}\\
Training samples & $10^{4}$ & $3{\times}10^{4}$ & $3{\times}10^{4}$ & $9.9{\times}10^{3}$ \\
Amplitude        & $1$ & $1$ & $\pm1$ & --- \\
Width            & $0.05$ & $[0.2,0.3]$ & $[0.1,0.2]$ & --- \\
Random shift     & $[0,1]$ & $[0,1]$ & $[0,1]$ & --- \\
\bottomrule
\end{tabular}
\end{table*}

\begin{table*}[h]
\centering
\setlength{\tabcolsep}{4pt}
\renewcommand{\arraystretch}{1.15}
\caption{Experimental details for KdV, NLS, and Burgers equations.}
\label{tab:exp-details-others}
\begin{tabular}{l ccc}
\toprule
& \textbf{KdV} & \textbf{NLS} & \textbf{Burgers} \\
\midrule
\multicolumn{4}{l}{\textit{Problem setup}}\\
Domain (periodic)    & $[0,2]$ & $[-20,80]$ & $[0,1]$ \\
Grid $N$ ($\Delta x$) & $200\,(0.01)$ & $1000\,(0.1)$ & $256\,(1/256)$ \\
Base scheme          & ZK leap-frog & Splitting & FTCS \\
Timestep $\Delta t$  & $5{\times}10^{-4}$ & $0.5$ & $4{\times}10^{-3}$ \\
Reference resolution     & $1000$ & $2{\times}10^{4}$ & $1024$ \\
Reference $\Delta t$     & $2{\times}10^{-6}$ & $10^{-4}$ & $0.01$ \\

\midrule
\multicolumn{4}{l}{\textit{Score network}}\\
Architecture         & \multicolumn{3}{c}{U-Net \citep{ronneberger2015u}} \\
Levels               & $3$ & $3$ & $4$ \\
Channels             & $(64,128,256)$ & $(64,128,256)$ & $(32,64,128,256)$ \\
Parameters           & $1.03$M & $1.03$M & $2.91$M \\
\midrule
\multicolumn{4}{l}{\textit{Training}}\\
Optimizer            & \multicolumn{3}{c}{Adam \citep{kingma2014adam}} \\
Learning rate        & $2{\times}10^{-4}$ & $2{\times}10^{-4}$ & $2{\times}10^{-4}$ \\
LR schedule          & Constant & Constant & Cosine \\
Batch size           & $256$ & $64$ & $512$ \\
Epochs               & $200$ & $1000$ & $300$ \\
Training-data resolution & $400$ & $10^{4}$ & $1024$ \\
Training-data $\Delta t$ & $5{\times}10^{-5}$ & $10^{-3}$ & $0.01$ \\

\midrule
\multicolumn{4}{l}{\textit{Correction (inference)}}\\
$\sigma$             & $0.05$ & $0.05$ & $0.1$ \\
Step $\eta$        & $3{\times}10^{-4}$ & $0.001$ & $0.01$ \\
Interval $K$         & $1$ & $1$ & $5$ \\
Enforced invariants  & Mass, $L^2$, $H$ & None & None \\

\midrule
\multicolumn{4}{l}{\textit{Initial conditions (training)}}\\
Training samples & $3.0{\times}10^{4}$ & $2.5{\times}10^{3}$ & $2.0{\times}10^{6}$ \\
Amplitude        & $[0.85,1.05]$ & $1$ & $[-3,3]$ \\
Width            & --- & $\sqrt{2}$ & --- \\
Random shift     & $[0,2]$ & $[-20,80]$ & --- \\

\bottomrule
\end{tabular}
\end{table*}

\end{document}